\documentclass[11pt,reqno]{amsart}
\usepackage[top=1.2in, bottom=1.2in, left=1.2in, right=1.2in]{geometry}
\usepackage{amsfonts, amssymb, amsmath, mathtools, dsfont}
\usepackage[linktoc=page, colorlinks, linkcolor=blue, citecolor=blue]{hyperref}
\usepackage{enumerate, commath}

\usepackage{graphicx}

\usepackage{subcaption} 
\usepackage[font=small]{caption} 
\numberwithin{equation}{section}

\DeclareMathOperator{\E}{\mathbb{E}}

\DeclareMathOperator{\Var}{Var}

\DeclareMathOperator{\Cov}{Cov}

\renewcommand{\Pr}[2][]{\mathbb{P}_{#1} \left\{ #2 \rule{0mm}{3mm}\right\}}

\newcommand{\ip}[2]{\langle#1,#2\rangle}

\def \P {\mathbb{P}}
\def \R {\mathbb{R}}

\def \FF {\mathcal{F}}

\def \e {\varepsilon}
\def \d {\delta}
\def \l {\lambda}
\def \s {\sigma}

\def \tran {\mathsf{T}}

\def \one {{\mathbf 1}}

\newtheorem{theorem}{Theorem}[section]
\newtheorem{proposition}[theorem]{Proposition}

\newtheorem{lemma}[theorem]{Lemma}

\theoremstyle{remark}

\begin{document}

\title{Memory capacity of neural networks with threshold and ReLU activations}

\author{Roman Vershynin}
\date{\today}

\address{Department of Mathematics, University of California, Irvine}
\email{rvershyn@uci.edu}

\thanks{Work in part supported by U.S. Air Force grant FA9550-18-1-0031}

\begin{abstract}
Overwhelming theoretical and empirical evidence shows 
that mildly over\-para\-metrized neural networks
-- those with more connections than the size of the training data -- 
are often able to memorize the training data with $100\%$ accuracy. 
This was rigorously proved for networks with sigmoid activation functions 
\cite{yamasaki1993lower, huang2003learning}
and, very recently, for ReLU activations \cite{yun2019small}. 
Addressing a 1988 open question of Baum \cite{baum1988capabilities},
we prove that this phenomenon holds for general multilayered perceptrons, 
i.e. neural networks with threshold activation functions, or with any mix 
of threshold and ReLU activations.
Our construction is probabilistic and exploits sparsity.
\end{abstract}

\maketitle


\section{Introduction}

This paper continues the long study of the memory capacity of neural architectures. How much information can a human brain learn? What are fundamental memory limitations, and how should the ``optimal brain'' be organized to achieve maximal capacity? These questions are complicated by the fact that we do not sufficiently understand the architecture of the human brain. But suppose that a neural architecture is known to us. Consider, for example, a given artificial neural network. Is there a general formula that expresses the memory capacity in terms of the network's architecture?

\subsection{Neural architectures}
In this paper we study a general layered, feedforward, fully connected neural architecture with arbitrarily many layers, arbitrarily many nodes in each layer, with either threshold or ReLU activation functions between all layers, and with the threshold activation function at the output node. 

Readers unfamiliar with this terminology may think of a neural architecture as a computational device that can compute certain compositions of linear and nonlinear maps. Let us describe precisely the functions computable by a neural architecture. Some of the best studied and most popular nonlinear functions $\phi: \R \to \R$, or ``activation functions'', include the {\em threshold function} and the {\em rectified linear unit} (ReLU), defined by
\begin{equation}	\label{eq: nonlinearity}
\phi(t) = \one_{\{t>0\}}
\quad \text{and} \quad
\phi(t) = \max(0,t) = t_+,
\end{equation}
respectively.\footnote{It should be possible to extend out results for other activation functions. 
To keep the argument simple, we shall focus on the threshold and ReLU nonlinearities in this paper.}  
We call a map {\em pseudolinear} if it is a composition of an affine transformation 
and a nonlinear transformation $\phi$ applied coordinate-wise.
Thus, $\Phi: \R^n \to \R^m$ is pseudolinear map if it can be expressed as 
$$
\Phi(x) = \phi(Vx-b), \quad x \in \R^n,
$$
where $V$ is a $m \times n$ matrix of ``weights'', $b \in \R^m$ is a vector of ``biases'', 
and $\phi$
is either the threshold or ReLU function \eqref{eq: nonlinearity}, 
which we apply to each coordinate of the vector $Wx-b$.

A {\em neural architecture} computes compositions of pseudolinear maps,
i.e. functions $F: \R^{n_1} \to \R$ of the type
$$
F = \Phi_L \circ \cdots \circ \Phi_2 \circ \Phi_1
$$
where $\Phi_1: \R^{n_1} \to \R^{n_2}$, $\Phi_2: \R^{n_2} \to \R^{n_3}$, 
\ldots, $\Phi_{L-1}: \R^{n_{L-1}} \to \R^{n_L}$, $\Phi_L: \R^{n_L} \to \R$
are pseudolinear maps. 
Each of maps $\Phi_i$ may be defined using either the threshold or ReLU function, 
and mix and match is allowed. However, for the purpose of this paper, 
{\em we require the output function $\Phi_L: \R^{n_L} \to \R$ to have the threshold
activation}.\footnote{General neural architectures used by practitioners and considered in the 
literature may have more than one output node and have other activation 
functions at the output node.}

We regard the matrices $V$ and $b$ in the definition of each pseudolinear map $\Phi_i$
as free parameters of the given neural architecture. Varying these free parameters
one can make a given neural architecture compute different functions $F: \R^{n_1} \to \{0,1\}$. 
Let us denote the class of such functions computable by a given architecture by
$$
\FF(n_1,\ldots,n_L,1).
$$

\begin{figure}[htp]	
  \centering 
    \includegraphics[width=0.35\textwidth]{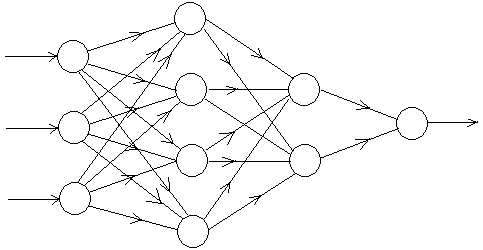} 
    \caption{A neural architecture with an input layer, 
    two hidden layers, and an output node. The class of functions
    $F: \R^3 \to \R$ this architecture can compute is denoted $\FF(3,4,2,1)$.}
  \label{fig: neural-network}
\end{figure}

A neural architecture can be visualized as a directed graph,
which consists of $L$ layers each having $n_i$ nodes (or {\em neurons}),
and one output node. 
Successive layers are connected by bipartite graphs, each of which 
represents a pseudolinear map $\Phi_i$.
Each neuron is a little computational device. 
It sums all inputs from the neurons in the previous layer 
with certain weights, applies the activation function $\phi$
to the sum, and passes the output to neurons in the next layer. 
More specifically, the neuron determines if the sum of incoming signals
from the previous layers exceeds a certain firing threshold $b$. 
If so, the neuron fires with either a signal of strength $1$ (if $\phi$ is 
the threshold activation function) or with strength proportional to 
the incoming signal (if $\phi$ is the ReLU activation).

\subsection{Memory capacity}

When can a given neural architecture remember a given data? 
Suppose, for example, that we have $K$ digital pictures of cats and dogs, encoded as
as vectors $x_1,\ldots,x_K \in \R^{n_1}$, and labels $y_1,\ldots,y_K \in \{0,1\}$
where $0$ stands for a cat and $1$ for a dog. Can we train a given neural architecture
to memorize which images are cats and which are dogs? 
Equivalently, does there exist a function
$F \in \FF(n_1,\ldots,n_L,1)$ such that 
\begin{equation}	\label{eq: memorization intro}
F(x_k)=y_k
\quad \text{for all} \quad 
k=1,\ldots,K?
\end{equation}

A common belief is that this should happen for any sufficiently {\em overparametrized} network --
an architecture that that has significantly more free parameters than the size of the training data. 
The free parameters of our neural architecture 
are the $n_{i-1} \times n_i$ weight matrices $V_i$ and the bias vectors $b_i \in \R^{n_i}$. 
The number of biases is negligible compared to the number of weights, 
and the number of free parameters is approximately the same as 
the {\em number of connections}\footnote{We dropped the number of output connections, which is negligible compared to $\overline{W}$.}
$$
\overline{W} = n_1n_2+\ldots+n_{L-1}n_L.
$$
Thus, one can wonder whether a general neural architecture is able to memorize the data as long as 
the number of connections is bigger than the size of the data, i.e. as long as
\begin{equation}	\label{eq: conj WK}
\overline{W} \gtrsim K.
\end{equation}

Motivated by this question, one can define {\em memory capacity} of a given architecture 
as the largest size of general data 
the architecture is able to memorize. In other words, the memory capacity 
is the largest $K$ such that for a general\footnote{One may sometimes wish to exclude 
some kinds of pathological data, so natural assumptions can be placed on the set of data points $x_k$. In this paper, for example, we consider unit and separated points $x_k$.} 
set of points $x_1,\ldots,x_K \in \R^{n_1}$ 
and for any labels $y_1,\ldots,y_K \in \{0,1\}$ 
there exists a function $F \in \FF(n_1,\ldots,n_L)$ 
that satisfies  \eqref{eq: memorization intro}.

The memory capacity is clearly bounded above by the vc-dimension,\footnote{This is because 
  the vc-dimension is the maximal $K$ for which {\em there exist} points $x_1,\ldots,x_K \in \R^{n_1}$ 
  so that for any labels $y_1,\ldots,y_K \in \{0,1\}$ there exists a function $F \in \FF(n_1,\ldots,n_L)$ 
that satisfies  \eqref{eq: memorization intro}. The memory capacity requires {\em any} general set 
of points $x_1,\ldots,x_K \in \R^{n_1}$ to succeed as above.}
which is $O(\overline{W} \log \overline{W})$ for neural architectures 
with threshold activation functions \cite{baum1989size}
and $O(L\overline{W} \log \overline{W})$
for neural architectures with ReLU activation functions \cite{bartlett2019nearly}.
Thus, our question is whether these bounds are tight -- is memory 
capacity (approximately) proportional to $\overline W$?

\subsection{The history of the problem}			\label{s: history}

A version of this question was raised by Baum \cite{baum1988capabilities} 
in 1988. 
Building on the earlier work of 
Cover \cite{cover1965geometrical},
Baum studied the memory capacity of {\em multilayer perceptrons},
i.e. feedforward neural architectures with threshold activation functions. 
He first looked at the network architecture $[n,m,1]$ with one hidden layer
consisting of $m$ nodes (and, as notation suggests, $n$ nodes in the hidden layer and one output node).
Baum noticed that for data points $x_k$ in general position in $\R^n$, the memory capacity of the architecture $[n,m,1]$ is about $nm$, i.e. it is proportional to the number of connections. 
This is not difficult: general position guarantees that the hyperplane spanned by any 
subset of $n$ data points misses any other data points; this allows one to train each of the $m$ neurons in the hidden layer on its own batch of $n$ data points.

Baum then asked if the same phenomenon persists for deeper neural networks.
He asked whether for large $K$ there exists a deep neural architecture 
with a total of $O(\sqrt{K})$ neurons in the hidden layers and with memory capacity at least $K$.
Such result would demonstrate the benefit of depth. Indeed, 
we just saw that shallow architecture 
$[n,O(\sqrt{K}),1]$ has capacity just $n \sqrt{K}$, which would be smaller
than the hypothetical capacity $K$ of deeper architectures for $n \ll K$. 

There was no significant progress on Baum's problem.
As Mitchison and Durbin noted in 1989, 
``one significant difference between a single threshold unit and a multilayer network is that, in the latter case, the capacity can vary between sets of input vectors, 
even when the vectors are in general position'' \cite{mitchison1989bounds}.
Attempting to count different functions $F$ that a deep network can realize on a given data set
$(x_k)$, Kowalczyk writes in 1997: ``One of the complications arising here is that in contrast to the single neuron case even for perceptrons with two hidden units, the number of implementable dichotomies may be different for various $n$-tuples in general position... Extension of this result to the multilayer case is still an open problem'' \cite{kowalczyk1997estimates}.


\medskip

The memory capacity problem is more tractable 
for neural architectures in which the threshold activation 
is replaced by one of its continuous proxies such as
ReLU, sigmoid, tanh, or polynomial activation functions.
Such activations allow neurons to fire with variable, controllable amplitudes. 
Heuristically, this ability makes it possible to encode the training data 
very compactly into the firing amplitudes. 

Yamasaki claimed without proof in 1993 that for the 
sigmoid activation $\phi(t) =1/(1+e^{-t})$ and for data in general position, 
the memory capacity of a general deep neural architecture is lower bounded 
$\overline{W}$, the number of connections \cite{yamasaki1993lower}.
A version of Yamasaki's claim was proved in 2003 by Huang for arbitrary data and 
neural architectures with two hidden layers \cite{huang2003learning}.

In 2016, Zhang et al. \cite{zhang2016understanding} 
gave a construction of an arbitrarily large (but not fully connected) neural architecture
with ReLU activations and whose memory capacity 
is proportional to both the number of connections
and the number of nodes. 
Hardt and Ma \cite{hardt2016identity} gave a different construction 
of a residual network with similar properties.

Very recently, Yun et at. \cite{yun2019small} removed the requirement that there be more
nodes than data, showing that the memory capacity of networks with ReLU and tahn activation functions is proportional to the number of connections. 
Ge et al. \cite{ge2019mildly} proved a similar result for 
polynomial activations.

Significant efforts was made in the last two years to justify why  
for overparametrized networks, the gradient descent and its variants could achieve $100\%$ capacity
on the training data \cite{du2018gradienta, du2018gradientb, li2018learning, zou2018stochastic, allen2018convergence, ji2019polylogarithmic, zou2019improved, oymak2019towards, song2019quadratic, arora2019fine, panigrahi2019effect}; see \cite{sun2019optimization} for a survey of related developments.


\subsection{Main result}

Meanwhile, the original problem of Baum \cite{baum1988capabilities} --
determine memory capacity of networks with threshold activations -- 
has remained open.
In contrast to the neurons with continuous activation functions,  
neurons with threshold activations either not fire at all 
of fire with the same unit amplitude. The strength of the incoming signal 
is lost when transmitted through such neurons, and it is not clear
how the data can be encoded.
 
This is what makes Baum's question hard. 
In this paper, we (almost) give a positive answer to this question.

Why ``almost''? First, the size of the input layer $n_1$ should not affect the capacity bound
and should be excluded from the count of the free parameters $\overline{W}$.
To see this, consider, for example, the data points $x_k \in \R^{n_1}$ all laying on one line; 
with respect to such data, the network is equivalent to one with $n_1=1$.
Next, ultra-narrow bottlenecks should be excluded at least for the threshold nonlinearity: 
for example, any layer with just $n_i=1$ node make the number of connections that occur
in the further layers irrelevant as free parameters.

In our actual result, we make somewhat stronger assumptions: 
in counting connections, we exclude not only the first layer but also the second;
we rule out all exponentially narrow bottlenecks (not just of size one); 
we assume that the data points $x_k$ are unit and separated;
finally, we allow logarithmic factors.

\begin{theorem}					\label{thm: main}
  Let $n_1,\ldots,n_L$ be positive integers, and set 
  $n_0 := \min(n_2,\ldots,n_L)$ and $n_\infty := \max(n_2,\ldots,n_L)$.
  Consider unit vectors $x_1,\ldots,x_K \in \R^n$ that satisfy
  \begin{equation}	\label{eq: main separation}
  \|x_i-x_j\|_2 \ge C \sqrt{\frac{\log\log n_\infty}{\log n_0}}.
  \end{equation}
  Consider any labels $y_1,\ldots,y_K \in \{0,1\}$.
  Assume that the number of deep connections $W := n_3 n_4 + \cdots + n_{L-1} n_L$
  satisfies
  \begin{equation}	\label{eq: W main}
  W \ge C K \log^5 K,
  \end{equation}
  as well as $K \le \exp(cn_0^{1/5})$ and $n_\infty \le \exp(cn_0^{1/5})$.
  Then the network can memorize the label assignment $x_k \to y_k$ exactly, 
  i.e. there exists a map $F \in \FF(n_1,\ldots,n_L,1)$ 
  such that
  \begin{equation}	\label{eq: main conclusion}
  F(x_k)=y_k
  \quad \text{for all} \quad 
  k=1,\ldots,K.
  \end{equation}
  Here $C$ and $c$ denote certain positive absolute constants.
\end{theorem}

In short, Theorem~\ref{thm: main} states that the memory capacity of a general
neural architecture with threshold or ReLU activations (or a mix thereof)
is lower bounded by the number of the deep connections.
This bound is independent of the depth, bottlenecks (up to exponentially narrow), 
or any other architectural details.

\subsection{Should the data be separated?}
One can wonder about the necessity of the separation assumption \eqref{eq: main separation}.
Can we just assume that $x_k$ are distinct? 
While this is true for ReLU and tanh activations \cite{yun2019small}, 
it is false for threshold activations. 
A moment's thought reveals that any pseudolinear map from $\R$ to $\R^{m}$  
transforms any line into a finite set such of cardinality $O(m)$. Thus, by pigeonhole 
principle, any map from layer $1$ to layer $2$ is non-injective on the set of $K$ 
data points $x_k$
-- which makes it impossible to memorize some label assignments -- unless $K = O(n_2)$. 
In other words, if we just assume that the data points $x_k$ are distinct,
the network must have at least as many {\em nodes} in the second layer 
as the number of data points. 
Still, the separation assumption \eqref{eq: separation} does not look tight and might be 
weakened.

\subsection{Related notions of capacity}
Instead of requiring the network to memorize the training data with $100\%$ accuracy
as in Theorem~\ref{thm: main}, one can ask to memorize just $1-\e$ fraction, or
just a half of the training data correctly. This corresponds to a relaxed, or  {\em fractional 
memory capacity} of neural architectures that was introduced by Cover in 1965 \cite{cover1965geometrical} and studied extensively afterwards. 

To estimate fractional capacity of a given architecture, one needs to count 
all functions $F$ this architecture can realize on a given finite set points $x_k$. 
When this set is the Boolean cube $\{0,1\}^n$, this amounts to counting
all Boolean functions $F: \{0,1\}^n \to \{0,1\}$ the architecture can realize.
The binary logarithm of the number of all such Boolean functions was called 
(expressive) capacity by Baldi and the author \cite{baldi2018neuronal, baldi2019capacity}.
For a neural architecture with all threshold activations and $L$ layers, 
the expressive capacity is equivalent
to the cubic polynomial in the sizes of layers $n_i$:
$$
\sum_{i=1}^{L-1} \min(n_1,\ldots,n_i)n_in_{i+1},
$$
up to an absolute constant factor \cite{baldi2019capacity}.
The factor $\min(n_1,\ldots,n_i)$ quantifies the effect of any bottlenecks 
that occur before layer $i$.

Similar results can be proved for the {\em restricted expressive capacity}
where we count the functions $F$ the architecture can realize on a 
given finite set of $K$ points $x_k$ \cite[Section~10.5]{baldi2019capacity}.
Ideally, one might hope to find that all $2^K$ functions 
can be realized on a general $K$-element set, which would imply 
that the memory capacity is at least $K$. However, the current 
results on restricted expressive capacity are not tight enough to reach 
such conclusions.

\section{The method}

Our construction of the function $F$ in Theorem~\ref{thm: main} is probabilistic. 
Let us first illustrate our approach for the architecture $[n,n,n,1]$ with two hidden layers,
and with threshold activations throughout. We would like to find a composition of 
pseudolinear functions 
$$
F: \R^n \xrightarrow{\Phi_1} \R^n \xrightarrow{\Phi_2} \R^n \xrightarrow{\Psi} \{0,1\}
$$
that fits the given data $(x_k, y_k)$ as in \eqref{eq: main conclusion}.

The first two maps $\Phi_1$ and $\Phi_2$ are {\em enrichment maps}
whose only purpose is spread the data $x_k$ in the space, 
transforming it into an almost orthogonal set. 
Specifically, $\Phi_1$ will transform the separated points $x_k$ into 
$o(1)$-orthogonal points $u_k$ (Theorem~\ref{thm: separated to eps-orthogonal}),
$\Phi_2$ will transform the $o(1)$-orthogonal points $u_k$
into $O(1/\sqrt{n})$-orthogonal points $v_k$ 
(Theorem~\ref{thm: eps-orthogonal to max orthogonal}),
and, finally, the {\em perception map} $\Psi(x)$ will fit the data: $\Psi(v_k) = y_k$.

\subsection{Enrichment}				\label{s: enrichment intro}

Our construction of the enrichment maps $\Phi_1$ and $\Phi_2$ exploits {\em sparsity}.
Both maps will have the form 
$$
\Phi(x) = \phi(Gx-\bar{b}) = \big( \one_{\{\ip{g_i}{x} > b\}} \big)_{i=1}^n
$$
where $G$ is an $n \times n$ Gaussian random matrix with all i.i.d. $N(0,1)$ coordinates,
$g_i \sim N(0,I_n)$ are independent standard normal random vectors, 
$\bar{b}$ is the vector whose all coordinates equal some value $b>0$.

If $b$ is large, $\Phi(x)$ is a sparse random vector with i.i.d. Bernoulli coordinates.
A key heuristic is that independent sparse random vectors are almost orthogonal. 
Indeed, if $u$ and $u'$ are independent random vectors in $\R^n$ whose 
all coordinates are $\textrm{Bernoulli(p)}$, then $\E \ip{u}{u'} = np^2$ while 
$\E\|u\|_2 = \E\|u'\|_2 = np$, so we should expect 
$$
\frac{\ip{u}{u'}}{\|u\|_2 \|u'\|_2} \sim p,
$$
making $u$ and $u'$ almost orthogonal for small $p$.

Unfortunately, the sparse random vectors $\Phi(x)$ and $\Phi(x')$ are not independent 
unless $x$ and $x'$ are exactly orthogonal. Nevertheless, our heuristic that 
sparsity induces orthogonality still works in this setting. To see this,
let us check that the correlation of the coefficients $\Phi(x)$ and $\Phi(x')$ is small
even if $x$ and $x'$ are far from being orthogonal. A standard asymptotic analysis 
of the tails of the normal distribution implies that
\begin{equation}	\label{eq: corr intro}
\E \Phi(x)_i \Phi(x')_i
= \Pr{ \ip{g}{x} > b, \, \ip{g}{x'} > b }
\le 2 \exp (-b^2 \d^2/8) \; \Pr{ \ip{g}{x} > b }
\end{equation}
if $x$ and $x'$ are unit and $\d$-separated (Proposition~\ref{prop: correlation separated}), 
and
\begin{equation}	\label{eq: eps intro}
\E \Phi(x)_i \Phi(x')_i
= \Pr{ \ip{g}{x} > b, \, \ip{g}{x'} > b }
\le 2 \exp(2b^2 \e) \; \big( \Pr{ \ip{g}{x} > b } \big)^2
\end{equation}
if $x$ and $x'$ are unit and $\e$-orthogonal (Proposition~\ref{prop: correlation epsilon-orthogonal}).

Now we choose $b$ so that the coordinates of $\Phi(x)$ and $\Phi(x')$ are sparse enough, i.e.
$$
\E \Phi(x)_i = \Pr{ \ip{g}{x} > b } = \frac{1}{\sqrt{n}} =: p;
$$
thus $b \sim \sqrt{\log n}$. 
Choose $\e$ sufficiently small to make the factor $\exp(2b^2 \e)$ in \eqref{eq: eps intro}
nearly constant, i.e.
$$
\e \sim \frac{1}{b^2} \sim \frac{1}{\log n}.
$$
Finally, we choose the separation threshold $\d$ sufficiently large 
to make the factor $2 \exp (-b^2 \d^2/8)$ in \eqref{eq: corr intro} less than $\e$, i.e.  
$$
\d \sim \sqrt{\frac{\log(1/\e)}{\log n}} 
\sim \sqrt{\frac{\log \log n}{\log n}};
$$
this explains the form of separation condition in Theorem~\ref{thm: main}.

With these choices, \eqref{eq: corr intro} gives
$$
\E \Phi(x)_i \Phi(x')_i \le \e p
$$
confirming our claim that $\Phi(x)$ and $\Phi(x')$ tend to be $\e$-orthogonal
provided that $x$ and $x'$ are $\d$-separated. 
Similarly, \eqref{eq: eps intro} gives 
$$
\E \Phi(x)_i \Phi(x')_i \lesssim p^2
$$
confirming our claim that $\Phi(x)$ and $\Phi(x')$ tend to be $(p = 1/\sqrt{n})$-orthogonal
provided that $x$ and $x'$ are $\e$-orthogonal.

\subsection{Perception}			\label{s: perception intro}

As we just saw, the enrichment process transforms our data points $x_k$
into $O(1/\sqrt{n})$-orthogonal vectors $v_k$. Let us now find a {\em perception map} 
$\Psi$ that can fit the labels to the data: $\Psi(v_k)=y_k$. 

Consider the random vector 
$$
w := \sum_{i=1}^K \pm y_i v_i
$$
where the signs are independently chosen with probability $1/2$ each. 
Then, separating the $k$-th term from the sum defining $w$ 
and assuming for simplicity that $v_k$ are unit vectors, we get
$$
\ip{w}{v_k} = \pm y_k + \sum_{i: \, i \ne k} \pm y_i \ip{v_i}{v_k} 
=: y_k + \textrm{noise}.
$$
Taking the expectation over independent signs, we see that
$$
\E (\textrm{noise})^2
= \sum_{i: \, i \ne k} y_i^2 \ip{v_i}{v_k}^2
$$
where $y_i^2 \in \{0,1\}$ and $\ip{v_i}{v_k}^2 = O(1/n)$ by almost orthogonality. 
Hence 
$$
\E (\textrm{noise})^2 \lesssim K/n = o(1)
$$
if $K \ll n$. This yields
$\ip{w}{v_k} = \pm y_k + o(1)$, or 
$$
|\ip{w}{v_k}| = y_k + o(1).
$$
Since $y_k \in \{0,1\}$, the ``mirror perceptron''\footnote{The mirror perceptron requires 
 not one but two neurons to implement, which is not a problem for us.}
$$
\Psi(v) := \one_{\{\ip{w}{v} > 1/2\}} + \one_{\{-\ip{w}{v} > 1/2\}}
$$ 
fits the data exactly: $\Psi(v_k)=y_k$.

\subsection{Deeper networks}			\label{s: assembly intro}

The same argument can be repeated for networks with variable sizes of layers,
i.e. $[n,m,d,1]$. Interestingly, the enrichment works fine even if $n \ll m \ll d$, 
making the lower-dimensional data almost orthogonal even in very high dimensions.
This explains why (moderate) bottlenecks -- narrow layers --
do not restrict memory capacity. 

The argument we outlined allows the network $[n,m,d,1]$ to fit around $d$ data points,
which is not very surprising, since we expect the memory capacity be proportional 
to the number of connections and not the number of nodes. However, the power of
enrichment allows us to boost the capacity using the standard method of 
batch learning (or distributed learning). 

Let us show, for example, how the network $[n,m,d,r,1]$ with three hidden layers can fit 
$K \sim dr$ data points $(x_k,y_k)$. 
Partition the data into $r$ batches each having $O(d)$ data points. 
Use our previous result to train each of the $r$ neurons in the fourth layer 
on its own batch of $O(d)$ points, while zeroing out all labels outside that batch.
Then simply sum up the results. (The details are found in Theorem~\ref{thm: shallow}.)

This result can be extended to deeper networks using {\em stacking}, 
or unrolling a shallow architecture into a deep architecture, 
thereby trading width for depth. Figure~\ref{fig: assembly} gives an illustration of stacking, and Theorem~\ref{thm: deep} provides the details. A similar stacking construction 
was employed in \cite{baldi2019capacity}. 

The success of stacking indicates that {\em depth has no benefit} formemorization 
purposes: a shallow architecture $[n,m,d,r,1]$ can memorize roughly as much data as any deep 
architecture with the same number of connections. 
It should be noted, however, that training algorithms commonly used by practitioners, 
i.e. variants of stochastic gradient descent, do not seem to lead to anything similar to stacking; 
this leaves the question of benefit of depth open.

\subsection{Neural networks as preconditioners}

As we explained in Section~\ref{s: enrichment intro}, the first two hidden layers of the network
act as preconditioners: they transform the input vectors $x_i$ 
into vectors $v_i$ that are almost orthogonal. 
Almost orthogonality facilitates memorization process in the deeper layers, as we saw in Section~\ref{s: perception intro}.

The idea to keep the data well-conditioned as it passes through the network
is not new. The learning rate of the stochastic gradient descent
(which we are not dealing with here) is related to how well conditioned is the
so-called gradient Gram matrix $H$. In the simplest scenario where 
the activation is ReLU and the network has one hidden layer of infinite size, 
$H$ is a $K \times K$ matrix with entries
$$
H_{ij} = \E \ip{x_i}{x_j} \one_{\{\ip{g}{x_i} > 0, \, \ip{g}{x_j} > 0\}}, 
\quad \text{where} \quad
g \sim N(0,I_{n_1}).
$$
Much effort was made recently to prove that $H$ is well-conditioned, 
i.e. its smallest singular value of $H$ is bounded away from zero, 
since this can be used to establishes a good convergence rate for the stochastic gradient 
descent, see the papers cited in Section~\ref{s: history} and especially \cite{allen2018convergence, du2018gradienta, du2018gradientb, panigrahi2019effect}. 
However, existing results that prove that $H$ is well conditioned 
only hold for very overparametrized networks, 
requiring at least $n_1^4$ nodes in the hidden layer \cite{panigrahi2019effect}. 
This is a much stronger 
overparametrization requirement than in our Theorem~\ref{thm: main}. 

On the other hand, as opposed to many of the results quoted above, Theorem~\ref{thm: main}
does {\em not} shed any light on the behavior of {\em stochastic gradient descent}, 
the most popular method for training deep networks. Instead of training the weights, 
we explicitly compute them from the data. 
This allows us to avoid dealing with the gradient Gram matrix:
our enrichment method provides an explicit way to make the data well conditioned. 
This is achieved by {\em setting the biases high enough} to enforce sparsity.
It would be interesting to see if similar preconditioning guarantees can be achieved
with small (or even zero) biases and thus without exploiting sparsity.

A different form of enrichment was developed recently in the paper \cite{baldi2019capacity}
which showed that a neural network can compute a lot of different Boolean functions. 
Toward this goal, an enrichment map was implemented in the first hidden layer. 
The objective of this map is to transform the input set (the Boolean cube $\{0,1\}^n$)
into a set $S \subset \{0,1\}^m$ on which there are lots of different threshold functions --
so that the next layers can automatically compute lots of different Boolean functions. 
While the general goal of this enrichment map in \cite{baldi2019capacity}
is the same as in the present paper --
achieve a more robust data representation that is passed to deeper layers --
the constructions of these two enrichment maps are quite different.

\subsection{Further observations and questions}

As we saw in the previous section, we utilized the first two hidden layers of the network 
to preprocess, or enrich, the data vectors $x_k$. This made us skip the sizes of the first two 
layers when we counted the number of connections $W$. 
If these vectors are already nice, no enrichment may be necessary, 
and we have a higher memory capacity.

Suppose, for example, that the data vectors $x_k$ in Theorem~\ref{thm: main} 
are $O(1/\sqrt{n_\infty})$-orthogonal. Then, since no enrichment is needed in this case, 
the conclusion of the theorem holds with
$$
W = n_1 n_2 + \cdots + n_{L-1} n_L,
$$
which is the sum of {\em all} connections in the network. 

If, on the other hand, the data vectors $x_k$ in Theorem~\ref{thm: main} 
are only $O(1/\sqrt{\log n_\infty})$-orthogonal, just the second enrichment is needed,
and so the conclusion of the theorem holds with
$$
W = n_2 n_3 + \cdots + n_{L-1} n_L,
$$
which is the sum of all connections between the non-input layers. 

This make us wonder: can enrichment be always achieved in one step instead of two? 
Can one find a pseudolinear map $\Phi: \R^n \to \R^n$ that transforms a given set of
$\d$-separated vectors $x_k$ (say, for $\d = 0.01$) into a set of 
$O(1/\sqrt{n})$-orthogonal vectors? If this is possible, we would not need to exclude the 
second layer from the parameter count, and Theorem~\ref{thm: main} would hold 
for $W := n_2 n_3 + \cdots + n_{L-1} n_L$.

A related question for further study is to find an optimal separation threshold $\d$ 
in the assumption $\|x_i - x_j\|_2 \ge \d$ in Theorem~\ref{thm: main}, 
and to remove the assumption that $x_k$ be unit vectors. Both the logarithmic separation 
level of $\d$ and the normalization requirement could be artifacts of the enrichment 
scheme we used. 

There are several ways Theorem~\ref{thm: main} could be extended. 
It should not be too difficult, for example, to allow the output layer have more than one node;
such {\em multi-output networks} are used in classification problems with multiple classes. 

Finally, it should be possible to extend the analysis for {\em completely general activation functions}. 
Threshold activations we treated are conceptually the hardest case, since
they act as extreme quantizers that restrict the flow of information through the network in the most dramatic way.

\subsection{The rest of the paper}			\label{s: rest of paper}
In Section~\ref{s: correlation decay} we prove bounds
\eqref{eq: corr intro} and \eqref{eq: eps intro} which 
control $\E \Phi(x)_i \Phi(x')_i$, the correlation of the coefficients
of the coordinates of $\Phi(x)$ and $\Phi(x')$. This immediately controls
the expected inner product $\E \ip{\Phi(x)}{\Phi(x')} = \sum_i \E \Phi(x)_i \Phi(x')_i$. 
In Section~\ref{s: deviation} we develop a deviation inequality to make sure 
that the inner product $\ip{\Phi(x)}{\Phi(x')}$ is close to its expectation with high probability. 
In Section~\ref{s: enrichment}, we take a union bound over all data points $x_k$ 
and thus control all inner products $\ip{\Phi(x_i)}{\Phi(x_j)}$ simultaneously. 
This demonstrates how enrichment maps $\Phi$ make the data almost orthogonal -- 
the property we outlined in Section~\ref{s: enrichment intro}.
In Section~\ref{s: perception}, we construct a random perception map $\Psi$ 
as we outlined in Section~\ref{s: perception intro}. 
We combine enrichment and perception in Section~\ref{s: assembly} as
we outlined in Section~\ref{s: assembly intro}. 
We first prove a version of our main result for networks with three hidden layers 
(Theorem~\ref{thm: shallow}); then we stack shallow networks into an arbitrarily 
deep architecture proving a full version of our main result in Theorem~\ref{thm: deep}.

\medskip

In the rest of the paper, positive absolute constants will be denoted $C, c, C_1, c_1$, etc.
The notation $f(x) \lesssim g(x)$ means that $f(x) \le C g(x)$ 
for some absolute constant $C$ and for all values of parameter $x$. 
Similarly, $f(x) \asymp g(x)$ means that $c g(x) \le f(x) \le C g(x)$ 
where $c$ is another positive absolute constant.

We call a map $E$ {\em almost pseudolinear} if $E(x) = \l \Phi(x)$ 
for some nonnegative constant $\l$ and some pseudolinear map $\Phi$. 
For the ReLU nonlinearity, almost pseudolinear maps are automatically pseudolinear, 
but for the threshold nonlinearity this is not necessarily the case.

\section{Correlation decay}			\label{s: correlation decay}

Let $g \sim N(0,I_n)$ and consider the random process
\begin{equation}	\label{eq: random process}
Z_x := \phi(\ip{g}{x} - b)
\end{equation}
which is indexed by points $x$ on the unit Euclidean sphere in $\R^n$. 
Here $\phi$ can be either the threshold of ReLU nonlinearity as in \eqref{eq: nonlinearity}, 
and $b \in \R$ is a fixed value. 
Due to rotation invariance, the correlation of $Z_x$ and $Z_{x'}$ 
only depends on the distance between $x$ and $x'$.  
Although it seems to be difficult to compute this dependence exactly,
we will demonstrate that the correlation of $Z_x$ and $Z_{x'}$ decays rapidly 
in $b$. We will prove this in two extreme regimes -- where 
$x$ and $x'$ are just a little separated, and where $x$ and $x'$ are almost orthogonal.

\subsection{Correlation for separated vectors}		\label{s: correlation separated}

Cauchy-Schwarz inequality gives a trivial bound
$$
\E Z_x Z_{x'} \le \E Z_x^2 
$$
with equality when $x=x'$. 
Our first result shows that if the vectors $x$ and $x'$ are $\d$-separated, this bound can be dramatically improved, and we have
$$
\E Z_x Z_{x'} \le 2\exp (-b^2 \d^2/8) \, \E Z_x^2.
$$

\begin{proposition}[Correlation for separated vectors]	\label{prop: correlation separated}
  Consider a pair of unit vectors $x, x' \in \R^n$, 
  and let $b \in \R$ be a number that is larger than a certain absolute constant.
  Then
  $$
  \E \phi \big( \ip{g}{x} - b \big) \, \phi \big( \ip{g}{x'} - b \big) 
  \le 2 \exp \bigg( -\frac{b^2 \norm[0]{x-x'}_2^2}{8} \bigg)
  \, \E \phi(\gamma - b)^2
  $$
  where $g \sim N(0,I_n)$ and $\gamma \sim N(0,1)$. 
\end{proposition}

\begin{proof}
{\em Step 1. Orthogonal decomposition.} 
Consider the vectors
$$
u := \frac{x+x'}{2}, \quad v := \frac{x-x'}{2}.
$$
Then $u$ and $v$ are orthogonal and 
$x= u+v$, $x' = u-v$.
We claim that
\begin{equation}	\label{eq: correlation deterministic}
\phi \big( \ip{z}{x} - b \big) \, \phi \big( \ip{z}{x'} - b \big) 
\le \big( \phi \big( \ip{z}{u} - b \big) \big)^2
\quad \text{for any } z \in \R^n.
\end{equation}

To check this claim, note that if both $\ip{z}{x}$ and $\ip{z}{x'}$ are greater than $b$
so is $\ip{z}{u}$. Expressing this implication as
\begin{equation}			\label{eq: separation rapid decay threshold}
\one_{\ip{z}{x}>b} \, \one_{\ip{z}{x'}>b} 
\le \one_{\ip{z}{u}>b},
\end{equation}
we conclude that \eqref{eq: correlation deterministic} 
holds for the threshold nonlinearity $\phi(t) = \one_{\{t>0\}}$.

To prove \eqref{eq: correlation deterministic} for the ReLU nonlinearity, note that
$$
\big( \ip{z}{x} - b \big) \, \big( \ip{z}{x'} - b \big) 
= \big( \ip{z}{u+v} - b \big) \, \big( \ip{z}{u-v} - b \big) 
= \big( \ip{z}{u} - b \big)^2 - \ip{z}{v}^2
\le \big( \ip{z}{u} - b \big)^2.
$$
Combine this bound with \eqref{eq: separation rapid decay threshold}
to get 
$$
\big( \ip{z}{x} - b \big) \, \big( \ip{z}{x'} - b \big) \, \one_{\ip{z}{x}>b} \, \one_{\ip{z}{x'}>b} 
\le \big( \ip{z}{u} - b \big)^2 \, \one_{\ip{z}{u}>b}.
$$
This yields \eqref{eq: correlation deterministic} for the ReLU nonlinearity $\phi(t) = t \, \one_{\{t>0\}}$.

\medskip

\noindent {\em Step 2. Taking expectation.}
Substitute $z = g \sim N(0,I_n)$ into the bound \eqref{eq: correlation deterministic}
and take expectation on both sides. We get
\begin{equation}	\label{eq: almost correlation}
\E \phi \big( \ip{g}{x} - b \big) \, \phi \big( \ip{g}{x'} - b \big) 
\le \E \phi \big( \ip{g}{u} - b \big)^2.
\end{equation}
Denote
$$
\d := \|v\|_2 = \frac{\|x-x'\|_2}{2}.
$$
Since $x = u+v$ is a unit vector and $u, v$ are orthogonal, we have $1 = \|u\|_2^2 + \|v\|_2^2$
and thus $\|u\|_2 = \sqrt{1-\d^2}$. Therefore, the random variable $\ip{g}{u}$ in the right side 
of \eqref{eq: almost correlation} is distributed identically with $\gamma \sqrt{1-\d^2}$
where $\gamma \sim N(0,1)$, 
and we obtain
\begin{equation}	\label{eq: factor still inside}
\E \phi \big( \ip{g}{x} - b \big) \, \phi\big( \ip{g}{x'} - b \big) 
\le \E \phi \big( \gamma \sqrt{1-\d^2} - b \big)^2.
\end{equation}

\medskip

\noindent {\em Step 3. Stability.}
Now use the stability property of the normal distribution, which we  
state in Lemma~\ref{lem: stability}. For $a=b$ larger than 
a suitable absolute constant, $z = -\d^2$, 
and for either the threshold or ReLU nonlinearity $\phi$, we see that
$$
\frac{\E \phi(\gamma \sqrt{1-\d^2} - b)^2}{\E \phi(\gamma - b)^2}
\le 2 \exp \bigg( -\frac{b^2 \d^2}{2(1-\d^2)} \bigg) (1-\d^2)^{3/2}
\le 2 \exp \bigg( -\frac{b^2 \d^2}{2} \bigg).
$$
Combine this with \eqref{eq: factor still inside} to complete the proof. 
\end{proof}

\subsection{Correlation for almost orthogonal vectors} \label{s: correlation almost orthogonal}

We continue to study the covariance structure of the random process $Z_x$. 
If $x$ and $x'$ are orthogonal, $Z_x$ and $Z_{x'}$ are independent and we have
$$
\E Z_x Z_{x'} = \big[ \E Z_x \big]^2.
$$
In this subsection, we show the stability of this equality. 
The result below implies that if $x$ and $x'$ are almost orthogonal, namely
$\abs{\ip{x}{x'}} \ll b^{-2}$, then
$$
\E Z_x Z_{x'} \lesssim \big[ \E Z_x \big]^2.
$$

\begin{proposition}[Correlation for $\e$-orthogonal vectors]	\label{prop: correlation epsilon-orthogonal}
  Consider a pair of vectors $u, u' \in \R^m$ satisfying 
  $$
  \abs[1]{\|u\|_2^2-1} \le \e, \quad  
  \abs[1]{\|u'\|_2^2-1} \le \e, \quad  
  \abs[1]{\ip{u}{u'}} \le \e 
  $$
  for some $\e \in (0,1/8)$.
  Let $b \in \R$ be a number that is larger than a certain absolute constant.
  Then  
  \begin{gather*}
  \E \phi \big( \ip{g}{u} - b \big)^2 
  	\ge \frac{1}{2} \exp(-b^2 \e) \, \E \phi(\gamma - b)^2; \\
  \E \phi \big( \ip{g}{u} - b \big) \, \phi \big( \ip{g}{u'} - b \big) 
  	\le 2 \exp(2b^2 \e) \, \big[ \E \phi(\gamma - b) \big]^2
  \end{gather*}
  where $g \sim N(0,I_m)$ and $\gamma \sim N(0,1)$.
\end{proposition}

In order to prove this proposition, we first establish a more general stability property:

\begin{lemma}			\label{lem: psi correlation}
  Let $\e \in (0,1/2)$ and let 
  $u$, $u'$, $g$, and $\gamma$ be as in Proposition~\ref{prop: correlation epsilon-orthogonal}.
  Then, for any measurable function $\psi : \R \to [0,\infty)$ we have
  $$
  \E \psi(\ip{g}{u}) \, \psi(\ip{g}{u'}) 
  \le \sqrt{\frac{1+2\e}{1-2\e}} \; \Big[ \E \psi \big( \gamma \sqrt{1+2\e} \big) \Big]^2.
  $$
\end{lemma}

\begin{proof}
Consider the $2 \times m$ matrix $A$
whose rows are $u$ and $u'$, and define the function 
$$
\Psi: \R^2 \to \R, \quad \Psi(x) := \psi(x_1) \, \psi(x_2).
$$
Since the vector $Ag$ has coordinates
$\ip{g}{u}$ and $\ip{g}{u'}$, we have $\Psi(Ag) = \psi(\ip{g}{u}) \, \psi(\ip{g}{u'})$.
Thus
\begin{equation} 	\label{eq: E psi psi}
\E \psi(\ip{g}{u}) \, \psi(\ip{g}{u'}) 
= \E \Psi(Ag)
= \frac{1}{2\pi \sqrt{\det(\Sigma)}} \int_{\R^2} \Psi(x) 
	\exp \bigg(-\frac{x^\tran \Sigma^{-1} x}{2}\bigg) \; dx
\end{equation}
where 
$$
\Sigma
= \Cov(Ag)
= AA^\tran = 
\begin{bmatrix}
  \|u\|_2^2 & \ip{u}{u'} \\
  \ip{u}{u'} & \|u'\|_2^2
\end{bmatrix}.
$$
The assumptions on $z,z'$ then give
\begin{equation}	\label{eq: remove Sigma}
\det(\Sigma) 
\ge 1 - 2\e 
\quad \text{and} \quad
x^\tran \Sigma^{-1} x 
\ge \frac{\|x\|_2^2}{1+2\e}
\quad \text{for all } x \in \R^2.
\end{equation}
Indeed, the first bound is straightforward.
To verify the second bound, note that each entry of the matrix $\Sigma-I_2$ 
is bounded in absolute value by $\e$. Thus, the operator norm of $\Sigma - I_2$ 
is bounded by $2\e$, which we can write as $\Sigma-I_2 \preceq 2\e I_2$ 
in the positive-semidefinite order. This implies that $\Sigma^{-1} \succeq (1+2\e)^{-1} I_2$,
and multiplying both sides of this relation by $x^\tran$ and $x$, we get
the second bound in \eqref{eq: remove Sigma}. 

Substitute \eqref{eq: remove Sigma} into \eqref{eq: E psi psi} to obtain
\begin{align*} 
\E \psi(\ip{g}{x}) \, \psi(\ip{g}{x'})
&\ge \frac{1}{2\pi \sqrt{1-2\e}} \int_{\R^2} \Psi(x) \exp \bigg(-\frac{\|x\|_2^2}{2(1+2\e)}\bigg) \, dx 
 = \sqrt{\frac{1+2\e}{1-2\e}} \E \Psi \big( h \sqrt{1+2\e} \big)
\end{align*}
where $h = (h_1,h_2) \sim N(0,I_2)$. 

It remains to recall that $\Psi(x) = \psi(x_1) \, \psi(x_2)$, so
$$
\E \Psi \big( h \sqrt{1+2\e} \big) 
= \E \psi \big( h_1 \sqrt{1+2\e} \big) \, \psi \big( h_2 \sqrt{1+2\e} \big)
= \Big[ \E \psi \big( \gamma \sqrt{1+2\e} \big) \Big]^2
$$
by independence.
Lemma~\ref{lem: psi correlation} is proved.
\end{proof}

\medskip

\begin{proof}[Proof of Proposition~\ref{prop: correlation epsilon-orthogonal}]
By assumption, $\|u\|_2 \ge \sqrt{1-\e}$, so 
\begin{equation}	\label{eq: gz-b factor inside}
\E \phi \big( \ip{g}{u} - b \big)^2
= \E \phi \big( \gamma \norm{u}_2 - b \big)^2
\ge \E \phi \big( \gamma \sqrt{1-\e} - b \big)^2,
\end{equation}
where the last inequality follows by monotonicity;
see Lemma~\ref{lem: monotonicity} for justification.
Now we use the stability property of the normal distribution 
that we state in Lemma~\ref{lem: stability}. 
For $a = b$ larger than a suitable absolute constant and $z = -\e$, 
it gives for both threshold of ReLU nonlinearities the following:
$$
\frac{\E \phi(\gamma \sqrt{1-\e} - b)^2}{\E \phi(\gamma - b)^2}
\ge 0.9 \exp \bigg( -\frac{b^2 \e}{2(1-\e)} \bigg) (1-\e)^{3/2}
\ge \frac{1}{2} \exp(-b^2 \e),
$$
where the last bound follows since $\e \le 1/8$.
Combining this with \eqref{eq: gz-b factor inside}, we obtain the first conclusion of the lemma.

Next, Lemma~\ref{lem: psi correlation} gives
\begin{equation}	\label{eq: epsilon is still there}
\E \phi \big( \ip{g}{u} - b \big) \, \big( \ip{g}{u'} - b \big) 
\le \sqrt{\frac{1+2\e}{1-2\e}} \; \Big[ \E \phi \big( \gamma \sqrt{1+2\e} - b \big) \Big]^2.
\end{equation}
Now we again use the stability property of the normal distribution, 
Lemma~\ref{lem: stability}, this time for $z = 2\e$. It gives
for both threshold of ReLU nonlinearities the following:
$$
\frac{\E \phi(\gamma \sqrt{1+2\e} - b)}{\E \phi(\gamma - b)}
\le 1.01 \exp \bigg( \frac{2b^2 \e}{2(1+2\e)} \bigg) (1+2\e)^{3/2}
\le 1.01 (1+2\e)^{3/2} \exp(b^2 \e).
$$
Combining this with \eqref{eq: epsilon is still there} gives
$$
\frac{\E \phi \big( \ip{g}{u} - b \big) \, \phi \big( \ip{g}{u'} - b \big)}{\big[ \E \phi(\gamma - b) \big]^2}
\le \sqrt{\frac{1+2\e}{1-2\e}} \Big( 1.01 (1+2\e)^{3/2} \exp(b^2 \e) \Big)^2 
\le 2 \exp(2 b^2 \e),
$$
where the last step follows since $\e \le 1/8$.
This completes the proof of Proposition~\ref{prop: correlation epsilon-orthogonal}.
\end{proof}

\section{Deviation}			\label{s: deviation}

In the previous section, we studied the covariance of the random process
$$
Z_x := \phi(\ip{g}{x} - b), \quad x \in \R^n,
$$
where $\phi$ is either the threshold of ReLU nonlinearity as in \eqref{eq: nonlinearity},
$g \sim N(0,I_n)$ is a standard normal random variable, 
and $b \in \R$ is a fixed value. 
Consider a multivariate version of this process,
a random pseudolinear map $\Phi: \R^n \to \R^m$
whose $m$ components are independent copies of $Z_x$. 
In other words, define
$$
\Phi(x) := \Big( \phi \big( \ip{g_i}{x} - b \big) \Big)_{i=1}^m
\quad \text{for } x \in \R^n,
$$
where $g_i \sim N(0,I_n)$ are independent standard normal random vectors.

We are interested in how the map $\Phi$ transforms the distances 
between different points. Since
$$
\E \ip{\Phi(x)}{\Phi(x')} = m \E Z_x Z_{x'}, 
$$
the bounds on $\E Z_x Z_{x'}$ we proved in the previous section
describe the behavior of $\Phi$ in expectation. 
In this section, we use standard concentration inequalities to ensure 
a similar behavior with high probability.

\begin{lemma}[Deviation]			\label{lem: deviation}
  Consider a pair of vectors $x, x' \in \R^n$ such that 
  $\|x\|_2 \le 2$, $\|x'\|_2 \le 2$, and let
  $b \in \R$ be a number that is larger than a certain absolute constant. Define
  \begin{equation}	\label{eq: p}
    p := \E \phi \big( \ip{g}{x} - b \big) \, \phi \big( \ip{g}{x'} - b \big), 
    \quad \text{where } g \sim N(0,I_n).
  \end{equation}
  Then for every $N \ge 2$, with probability at least $1-2mN^{-5}$ we have
  $$
  \big| \ip{\Phi(x)}{\Phi(x')} - mp \big| 
  \le C_1 \big( \sqrt{mp} \, \log N + \log^2 N \big).
  $$
\end{lemma}

\begin{proof}
{\em Step 1. Decomposition and truncation.}
By construction, $\E \ip{\Phi(x)}{\Phi(x')} = mp$.
The deviation from the mean is
\begin{equation}	\label{eq: ip as sum}
\ip{\Phi(x)}{\Phi(x')} - mp
= \sum_{i=1}^m \phi \big( \gamma_i - b \big) \, \phi \big( \gamma'_i - b \big) - mp
\end{equation}
where $\gamma_i := \ip{g_i}{x}$ and $\gamma'_i := \ip{g_i}{x'}$. 
These two normal random variables are possibly correlated, 
and each has zero mean and variance bounded by $4$.

We will control the sum of i.i.d. random variables in \eqref{eq: ip as sum} 
using Bernstein's concentration inequality.
In order to apply it, we first perform a standard truncation of the terms of the sum. 
The level of truncation will be 
\begin{equation}	\label{eq: truncation level}
M := C_2 \sqrt{\log N}
\end{equation}
where $C_2$ is a sufficiently large absolute constant.
Consider the random variables
\begin{align*}
Z_i &:= \phi(\gamma_i - b) \, \phi(\gamma'_i - b) \,
	\one_{\{ \gamma_i \le M \text{ and } \gamma'_i \le M\}}, \\
R_i &:= \phi(\gamma_i - b) \, \phi(\gamma'_i - b) \,
	\one_{\{ \gamma_i > M \text{ or } \gamma'_i > M\}}.
\end{align*}
Then we can decompose the sum in \eqref{eq: ip as sum} as follows:
\begin{equation}	\label{eq: sum decomposition}
\ip{\Phi(x)}{\Phi(x')} - mp
= \sum_{i=1}^m (Z_i - \E Z_i) + \sum_{i=1}^m (R_i - \E R_i).
\end{equation}

\noindent {\em Step 2. The residual is small.}
Let us first control the residual, i.e. the second sum on the right side of \eqref{eq: sum decomposition}. 
For a fixed $i$, the probability that $R_i$ is nonzero can be bounded by
\begin{equation}	\label{eq: Ri ne 0}
\Pr{\gamma_i > M \text{ or } \gamma'_i > M} 
\le \Pr{\gamma_i > M} + \Pr{\gamma'_i > M}
\le 2 \Pr{\gamma > M/2}
\le N^{-10}
\end{equation}
where $\gamma \sim N(0,1)$.
In the second inequality, we used that $\gamma_i$ and $\gamma'_i$ 
are normal with mean zero and variance at most $4$. 
The third inequality follows from the asymptotic 
\eqref{eq: Mills ratio tail}
on the tail of the normal distribution
and our choice \eqref{eq: truncation level} of 
the truncation level $M$ with sufficiently large $C_0$.

Taking the union bound we see that all $R_i$ vanish simultaneously 
with probability at least $1-mN^{-10}$.
Furthermore, by monotonicity,
\begin{align}
\E R_i 
&\le \E \phi(\gamma_i) \, \phi(\gamma'_i) \,
	\one_{\{ \gamma_i > M \text{ or } \gamma'_i > M\}}
	\nonumber\\
&\le \big( \E \phi(\gamma_i)^4 \big)^{1/4} \, \big( \E \phi(\gamma'_i)^4 \big)^{1/4} \,
	\big( \Pr{\gamma_i > M \text{ or } \gamma'_i > M} \big)^{1/2}	
	\label{eq: ERi}
\end{align}
where in the last step we used generalized H\"older's inequality.
Now, for the threshold nonlinearity $\phi(t) = \one_{\{t >0\}}$, 
the terms $\E \phi(\gamma_i)^4$ and $\E \phi(\gamma'_i)^4$ obviously equal $1/2$, 
and for the ReLU nonlinearity $\phi(t) = t_+$ these terms 
are bounded by the fourth moment of the standard normal distribution, 
which equals $3$. Combining this with \eqref{eq: Ri ne 0} gives
$$
0 \le \E R_i \le 2N^{-5}.
$$
Summarizing, with probability at least $1-mN^{-10}$, we have
\begin{equation}	\label{eq: residual sum done}
\bigg| \sum_{i=1}^m (R_i - \E R_i) \bigg|
= \bigg| \sum_{i=1}^m \E R_i \bigg|
\le 2mN^{-5} \le 1. 
\end{equation}
The last bound holds because otherwise we have $1 - 2mN^{-5}<0$ and
the statement of the proposition holds trivially.

\medskip

\noindent {\em Step 3. The main sum is concentrated.}
To bound the first sum in \eqref{eq: sum decomposition}, we can use Bernstein's 
inequality \cite{Vbook}, which we can state as follows. If $Z_1,\ldots,Z_m$ are independent 
random variables and $s \ge 0$, then with probability at least $1-2e^{-s}$ we have
\begin{equation}	\label{eq: Bernstein}
\bigg| \sum_{i=1}^m (Z_i - \E Z_i) \bigg|
\lesssim \sigma \sqrt{s} + K s,
\end{equation}
where $\s^2 = \sum_{i=1}^m \Var(Z_i)$ and $K=\max_i \|Z_i\|_\infty$.
In our case, it is easy to check that for both threshold and ReLU nonlinearities $\phi$, 
we have 
$$
K = \|Z_1\|_\infty \le \phi(M-b)^2 \le M^2,
$$
and 
$$
\s^2 
= m \Var(Z_1)
\le m \E Z_1^2
\le M^2 m \E Z_1
\le M^2 m \E \phi (\gamma_1 - b) \, \phi (\gamma'_1 - b)
= M^2 m p
$$
by definition of $p$ in \eqref{eq: p}.
Apply Bernstein's inequality \eqref{eq: Bernstein} for
\begin{equation}	\label{eq: Bernstein s}
s = C_3 \log N
\end{equation}
where $C_3$ is a suitably large absolute constant. 
We obtain that with probability at least $1-2e^{-s} \ge 1-N^{-10}$, 
\begin{equation}	\label{eq: main sum done}
\bigg| \sum_{i=1}^m (Z_i - \E Z_i) \bigg|
\lesssim \sqrt{M^2 m p s} + M^2 s
\lesssim \sqrt{mp} \, \log N + \log^2 N.
\end{equation}
Here we used the choice of $M$ we made in \eqref{eq: truncation level}
and $s$ in \eqref{eq: Bernstein s}.

Combining the bounds on the residual \eqref{eq: residual sum done}
and on the main sum \eqref{eq: main sum done} and putting them into 
the decomposition \eqref{eq: sum decomposition}, we conclude that
with probability at least $1-2mN^{-10}$,
$$
\big| \ip{\Phi(x)}{\Phi(x')} - mp \big| 
  \lesssim \sqrt{mp} \, \log N + \log^2 N + 1
  \lesssim \sqrt{mp} \, \log N + \log^2 N.
$$
The proof is complete.
\end{proof}

\section{Enrichment}				\label{s: enrichment}

In the previous section, we defined a random pseudolinear map 
\begin{equation}	\label{eq: Phi nm}
\Phi: \R^n \to \R^m, \quad \Phi(x) := \Big( \phi \big( \ip{g_i}{x} - b \big) \Big)_{i=1}^m,
\end{equation}
where $\phi$ is either the threshold of ReLU nonlinearity as in \eqref{eq: nonlinearity},
$g_i \sim N(0,I_n)$ are independent standard normal random vectors,
and $b$ is a fixed value. 

We will now demonstrate the ability of $\Phi$ to ``enrich'' the data, to move different points away from each other. To see why this could be the case, 
choose the value of $b$ to be moderately large, say $b = 100 \sqrt{\log m}$.
Then with high probability, most of the random variables $\ip{g_i}{x}$ will fall below $b$,
making most of the coordinates $\Phi(x)$ equal zero, thus making $\Phi(x)$ a random sparse vector. 
Sparsity will tend to make $\Phi(x)$ and $\Phi(x')$ almost orthogonal even when $x$ and $x'$ are just 
a little separated from each other. 

To make this rigorous, we can use the results of Section~\ref{s: correlation decay}
to check that for such $b$, the coordinates of $\Phi(x)$ and $\Phi(x')$ 
are almost uncorrelated. This immediately implies that $\Phi(x)$ and $\Phi(x')$ are almost orthogonal 
in expectation, and the deviation inequality from Section~\ref{s: deviation} then implies 
that the same holds with high probability. This allows us to take a union bound over all data points $x_i$ and conclude that $\Phi(x_i)$ and $\Phi(x_j)$ are almost orthogonal for all distinct data points.

As in Section~\ref{s: correlation decay}, we will prove this in two regimes, first 
for the data points that are just a little separated, and then for the data points that are almost orthogonal.

\subsection{From separated to $\e$-orthogonal}

In this part we show that the random pseudolinear map $\Phi$ transforms 
separated data points into almost orthogonal points. 

\begin{lemma}[Enrichment I: from separated to $\e$-orthogonal]	\label{lem: separated to eps-orthogonal}
  Consider a pair of unit vectors $x, x' \in \R^n$ satisfying
  \begin{equation}	\label{eq: separation}
  \|x-x'\|_2 \ge C_2 \sqrt{\frac{\log(1/\e)}{\log m}}
  \end{equation}
  for some $\e \in [m^{-1/5},1/8]$. 
  Let $2 \le N \le \exp(m^{1/5})$, and let $p$ and $b$ be numbers such that   
  $$
  p = \frac{C_2 \log^2 N}{\e^2 m} =\E \phi(\gamma-b)^2.  
  $$
  Consider the random pseudolinear map $\Phi: \R^n \to \R^m$ defined in \eqref{eq: Phi nm}.
  Then with probability at least $1-4mN^{-5}$, the vectors 
  $$
  u := \frac{\Phi(x)}{\sqrt{mp}}, \quad u' := \frac{\Phi(x')}{\sqrt{mp}}
  $$
  satisfy 
  $$
  \abs[1]{\|u\|_2^2-1} \le \e, \quad   
  \abs[1]{\ip{u}{u'}} \le \e. 
  $$
\end{lemma}

\begin{proof}
{\em Step 1. Bounding the bias $b$.}
We begin with some easy observations. 
Note that $\|x-x'\|_2$ is bounded above by $2$
and below by $C_2/\sqrt{\log m}$.
Thus, by setting the value of $C_2$ sufficiently large, 
we can assume that $m$ is arbitrarily large, i.e. 
larger than any given absolute constant. 
Furthermore, the restrictions on $\e$ and $N$ yield
\begin{equation}	\label{eq: p bounds}
m^{-1} \le p \lesssim m^{-1/10},
\end{equation}
so $p$ is arbitrarily small, smaller than any given absolute constant. 
The function $t \mapsto \E \phi(\gamma-t)^2$ is continuous,
takes an absolute constant value at $t=0$, and tends to zero as $t \to \infty$. 
Thus the equation $\E \phi(\gamma-t)^2 = p$ has a solution, 
so $b$ is well defined and $b \ge 1$.

To get a better bound on $b$, one can use \eqref{eq: Mills ratio tail} for the threshold
nonlinearity and Lemma~\ref{lem: ReLU normal} for ReLU, which give
$$
\log \E \phi(\gamma-b)^2 \asymp -b^2.
$$
Since $p = \E \phi(\gamma-b)^2$, this and \eqref{eq: p bounds} yield
\begin{equation}	\label{eq: b bounds}
b \asymp \sqrt{\log m}.
\end{equation}

\noindent {\em Step 2. Controlling the norm.}
Applying Lemma~\ref{lem: deviation} for $x=x'$, we obtain 
with probability at least $1-2mN^{-5}$ that
$$
\big| \|\Phi(x)\|_2^2 - mp \big| 
\le C_1 \big( \sqrt{mp} \, \log N + \log^2 N \big).
$$
Divide both sides by $mp$ to get
$$
\abs[1]{\|u\|_2^2-1} 
\le C_1 \bigg( \frac{\log N}{\sqrt{mp}} + \frac{\log^2 N}{mp} \bigg)
\le \e,
$$
where the second inequality follows from our choice of $p$ with large $C_2$.
We proved the first conclusion of the proposition.

\medskip

\noindent {\em Step 3. Controlling the inner product.}
Proposition~\ref{prop: correlation separated} gives
\begin{equation}	\label{eq: q bound}
q := \E \phi \big( \ip{g}{x} - b \big) \, \phi \big( \ip{g}{x'} - b \big) 
\le 2 \exp \bigg( -\frac{b^2 \norm[0]{x-x'}_2^2}{8} \bigg) p
\le \e^{10} p,
\end{equation}
where in the last step we used the bounds \eqref{eq: b bounds} on $b$ 
and the separation assumption \eqref{eq: separation} with 
a sufficiently large constant $C_2$.
Now, applying Lemma~\ref{lem: deviation}, we obtain 
with probability at least $1-2mN^{-5}$ that
$$
\abs{\ip{\Phi(x)}{\Phi(x')}}
\le mq + C_1 \big( \sqrt{mq} \, \log N + \log^2 N \big).
$$
Divide both sides by $mp$ to obtain
$$
\abs{\ip{u}{u'}} \le \frac{q}{p} + C_1 \bigg( \frac{\sqrt{q} \log N}{\sqrt{m} p} + \frac{\log^2 N}{mp} \bigg)
\le \e^2,
$$
where the last step follows from the bound \eqref{eq: q bound} on $q$ and our choice of $p$
with a sufficiently large $C_2$.
This is an even stronger bound than we claimed. 
\end{proof}

\begin{theorem}[Enrichment I: from separated to $\e$-orthogonal]	\label{thm: separated to eps-orthogonal}
  Consider unit vectors $x_1,\ldots,x_K \in \R^n$ that satisfy  
  $$
  \|x_i-x_j\|_2 \ge C_2 \sqrt{\frac{\log(1/\e)}{\log m}}
  $$
  for all distinct $i,j$, where $\e \in [m^{-1/5},1/8]$. 
  Assume that $K \le \exp(c_2 m^{1/5})$.
  Then there exists an almost\footnote{Recall from Section~\ref{s: rest of paper} that an almost pseudolinear map $E$ is, by definition, a pseudolinear map multiplied by a nonnegative constant.
  In our case, $E = (mp)^{-1/2} \Phi$.} 
  pseudolinear map $E: \R^n \to \R^m$ such that 
  the vectors $u_k := E(x_k)$ satisfy
  $$
  \abs[1]{\|u_i\|_2^2-1} \le \e, \quad   
  \abs[1]{\ip{u_i}{u_j}} \le \e
  $$
  for all distinct indices $i,j = 1,\ldots,K$.
\end{theorem}

\begin{proof}
Apply Lemma~\ref{lem: separated to eps-orthogonal} followed by a union 
bound over all pairs of distinct vectors $x_k$. If we chose $N = 2mK$, then the probability 
of success is at least $1 - K^2 \cdot 4m (2mK)^{-5} > 0$. The proof is complete.
\end{proof}

\subsection{From $\e$-orthogonal to $\frac{1}{\sqrt{d}}$-orthogonal}

In this part we show that a random pseudolinear map $\Phi: \R^m \to \R^d$ 
makes almost orthogonal data points even closer to being orthogonal:
$\Phi$ reduces the inner products from a small constant $\e$ to $O(1/\sqrt{d})$.

The pseudolinear map $\Phi$ considered in this part will have the same form 
as in \eqref{eq: Phi nm}, but for different dimensions:
\begin{equation}	\label{eq: Phi md}
\Phi: \R^m \to \R^d, \quad \Phi(u) := \Big( \phi \big( \ip{g_i}{u} - b \big) \Big)_{i=1}^m,
\end{equation}
where $\phi$ is either the threshold of ReLU nonlinearity as in \eqref{eq: nonlinearity},
$g_i \sim N(0,I_m)$ are independent standard normal random vectors,
and $b$ is a fixed value. 

\begin{lemma}[Enrichment II: from $\e$-orthogonal to $\frac{1}{\sqrt{d}}$-orthogonal]	\label{lem: eps-orthogonal to max orthogonal}
  Consider a pair of vectors $u, u' \in \R^m$ satisfying
  \begin{equation}	\label{eq: almost orthonormal}
  \abs[1]{\|u\|_2^2-1} \le \e, \quad  
  \abs[1]{\|u'\|_2^2-1} \le \e, \quad  
  \abs[1]{\ip{u}{u'}} \le \e 
  \end{equation}
  for some $0 < \e \le c_3/\log d$. 
  Let $2 \le N \le \exp(c_3d^{1/5})$, and let $p$ and $b$ be numbers such that   
  $$
  p := \frac{1}{\sqrt{d}}, \qquad \E \phi(\gamma-b)^2 = p.
  $$
  Then with probability at least $1-4dN^{-5}$, the vectors 
  $$
  v := \frac{\Phi(u)}{\sqrt{dp}}, \quad v' := \frac{\Phi(u')}{\sqrt{dp}}
  $$
  satisfy 
  $$
  \|v\|_2 \ge \frac{1}{2}, \quad   
  \abs[1]{\ip{v}{v'}} \le \frac{C_3(\log d + \log^2 N)}{\sqrt{d}}. 
  $$
\end{lemma}

\begin{proof}
{\em Step 1. Bounding the bias $b$.}
Following the beginning of the proof of Lemma~\ref{lem: separated to eps-orthogonal}, 
we can check that $b$ exists and
\begin{equation}	\label{eq: b bounds II}
b \asymp \sqrt{\log d}.
\end{equation}

\noindent {\em Step 2. Controlling the norm.}
Applying Proposition~\ref{prop: correlation epsilon-orthogonal}, we see that
$$
p_0 := \E \phi \big( \ip{g}{u} - b \big)^2 
\ge \frac{1}{2} \exp(-b^2 \e) p
\ge \frac{p}{3},
$$
where in the last step we used the bound \eqref{eq: b bounds II} on $b$
and the assumption on $\e$ with a sufficiently small constant $c_3$.
Then, applying Lemma~\ref{lem: deviation} for $x=x'=u$, we obtain 
with probability at least $1-2dN^{-5}$ that
$$
\|\Phi(u)\|_2^2 
\ge d p_0 - C_1 \big( \sqrt{dp_0} \, \log N + \log^2 N \big)
\ge \frac{3}{4} dp_0
\ge \frac{1}{4} dp,
$$
where we used our choice of $p$ and the restriction on $N$
with sufficiently small constant $c_3$.
Divide both sides by $dp$ to get 
$$
\|v\|_2 \ge \frac{1}{2},
$$ 
which is the first conclusion of the proposition.
 
\medskip

\noindent {\em Step 3. Controlling the inner product.}
Proposition~\ref{prop: correlation epsilon-orthogonal} gives
\begin{equation}	\label{eq: q begin}
q := \E \phi \big( \ip{g}{u} - b \big) \, \phi \big( \ip{g}{u'} - b \big) 
\le 2 \exp(2b^2 \e) \, \big[ \E \phi(\gamma - b) \big]^2
\lesssim 	\big[ \E \phi(\gamma - b) \big]^2,
\end{equation}
where the last inequality follows as before from
bound \eqref{eq: b bounds II} on $b$
and the assumption on $\e$ with sufficiently small $c_3$.

Next, we will use the following inequality that holds for all sufficiently large $a>0$:
$$
\E \phi(\gamma - a) \le a \cdot \E \phi(\gamma - a)^2.
$$
For the threshold nonlinearity $\phi$, this bound is trivial even without the factor $a$ in the right side.
For the ReLU nonlinearity, it follows from Lemma~\ref{lem: ReLU normal} in the Appendix.
Therefore, we have 
$$
\E \phi(\gamma - b) \le b p \lesssim p \sqrt{\log d}
$$
where we used \eqref{eq: b bounds II} in the last step. 
Substituting this into \eqref{eq: q begin}, we conclude that
\begin{equation}	\label{eq: q bound II}
q \lesssim p^2 \log d.
\end{equation}

Now, applying Lemma~\ref{lem: deviation}, we obtain 
with probability at least $1-2mN^{-5}$ that
$$
\abs{\ip{\Phi(u)}{\Phi(u')}}
\lesssim dq + \sqrt{dq} \, \log N + \log^2 N.
$$
Divide both sides by $dp$ to obtain
$$
\abs{\ip{v}{v'}} \lesssim \frac{q}{p} + \frac{\sqrt{q} \log N}{\sqrt{d} p} + \frac{\log^2 N}{dp}
\lesssim \frac{\log d + \log^2 N}{\sqrt{d}}. 
$$
where in the last step we used \eqref{eq: q bound II} and our choice of $p$.
\end{proof}

\begin{theorem}[Enrichment II: from $\e$-orthogonal to $\frac{1}{\sqrt{d}}$-orthogonal]	\label{thm: eps-orthogonal to max orthogonal}
  Consider vectors $u_1,\ldots,u_K \in \R^n$ that satisfy 
  $$
  \abs[1]{\|u_i\|_2^2-1} \le \e, \quad  
  \abs[1]{\ip{u_i}{u_j}} \le \e
  $$
  for all distinct $i,j$, where $0 < \e \le c_3/\log d$. 
  Assume that $K \le \exp(c_3d^{1/5})$.  
  Then there exists an almost pseudolinear map $R: \R^m \to \R^d$ such that 
  the vectors $v_k := R(u_k)$ satisfy
  $$
  \|v_i\|_2 \ge \frac{1}{2}, \quad   
  \abs[1]{\ip{v_i}{v_j}} \le \frac{C_4 \log^2 (dK)}{\sqrt{d}}
  $$
  for all distinct indices $i,j = 1,\ldots,K$.
\end{theorem}

\begin{proof}
Apply Lemma~\ref{lem: eps-orthogonal to max orthogonal} followed by a union 
bound over all pairs of distinct vectors $u_k$. If we chose $N = 2dK$, then the probability 
of success is at least $1 - K^2 \cdot 4d (2dK)^{-5}>0$. The proof is complete.
\end{proof}

\section{Perception}				\label{s: perception}

The previous sections were concerned with preprocessing, or ``enrichment'', of the data. 
We demonstrated how a pseudolinear map can transform the original 
data points $x_k$, which can be just a little separated, into $\e$-orthogonal points $u_k$
with $\e = o(1)$, and further into $\eta$-orthogonal points $v_k$ with $\eta=O(1/\sqrt{d})$.
In this section we train a pseudolinear map that can memorize
any label assignment $y_k$ for the $\eta$-orthogonal points $v_k$. 

We will first try to train a single neuron to perform this task assuming 
that the number $K$ of the data points $v_k$ is smaller than the dimension $d$, 
up to a logarithmic factor. Specifically, we construct a vector $w \in \R^n$
so that the values $\abs{\ip{w}{v_k}}$ are small whenever $y_k=0$ and 
large whenever $y_k=1$. Our construction is probabilistic: we choose
$w = \sum_{k=1}^K \pm y_k v_k$ with random independent signs, 
and show that $w$ succeeds with high probability. 

\begin{lemma}[Perception]			\label{lem: perception 1}
  Let $\eta \in (0,1)$ and 
  consider vectors $v_1,\ldots,v_K \in \R^d$ satisfying
  \begin{equation}		\label{eq: perception assumptions}
  \|v_i\|_2 \ge \frac{1}{2}, \quad \abs{\ip{v_i}{v_j}} \le \eta
  \end{equation}
  for all distinct $i,j$. 
  Consider any labels $y_1,\ldots,y_K \in \{0,1\}$, at most $K_1$ of which equal $1$. 
  Assume that
  $K_1 \log K \le c_4 \eta^{-2}$.
  Then there exists a vector $w \in \R^d$ that satisfies 
  the following holds for every $k=1,\ldots,K$:
  \begin{equation}	\label{eq: perception}
  \abs{\ip{w}{v_k}} \le \frac{1}{16} \text{ if $y_k=0$}; \qquad
  \abs{\ip{w}{v_k}} \ge \frac{3}{16} \text{ if $y_k=1$}.
  \end{equation}
\end{lemma}

\begin{proof}
Let $\xi_1,\ldots,\xi_K$ be independent Rademacher random variables and define
$$
w := \sum_{k=1}^K \xi_k y_k v_k.
$$
We are going to show that the random vector $w$ satisfies the conclusion of the proposition 
with positive probability.

Let us first check the conclusion \eqref{eq: perception} for $k=1$. 
To this end, we decompose $\ip{w}{v_1}$ 
as follows:
$$
\ip{w}{v_1} = \xi_1 y_1 \|v_1\|_2^2 + \sum_{k=2}^K \xi_k y_k \ip{v_k}{v_1}
=: \textrm{signal} + \textrm{noise}.
$$
To bound the noise, we shall use Hoeffding's inequality (see e.g. \cite[Theorem~2.2.2]{Vbook}), 
which can be stated as follows.
If $a_1,\ldots,a_N$ are any fixed numbers and $s \ge 0$, then with probability at least $1-2e^{-s^2/2}$ we have
$$
\abs[3]{\sum_{k=1}^N \xi_k a_k}
\le s \bigg( \sum_{k=1}^N a_k^2 \bigg)^{1/2}.
$$
Using this for $s = 4\sqrt{\log K}$, we conclude that with probability at least $1-2K^{-8}$, we have
$$
\abs{\textrm{noise}} 
\le 4\sqrt{\log K} \bigg( \sum_{k=2}^K y_k^2 \ip{v_k}{v_1}^2 \bigg)^{1/2}
\le 4\sqrt{\log K} \, \sqrt{K_1} \eta \le \frac{1}{16},
$$
where we used \eqref{eq: perception assumptions} and 
the assumption on $K, K_1$ with a sufficiently small constant $c_4$.

If $y_1=0$, the signal is zero and so $\abs{\ip{w}{v_1}} = \abs{\textrm{noise}} \le 1/16$, as claimed. If $y_1=1$ then $\abs{\textrm{signal}} = \|v_1\|_2^2 \ge 1/4$ and thus
$$
\abs{\ip{w}{v_1}} 
\ge \abs{\textrm{signal}} - \abs{\textrm{noise}}
\ge \frac{1}{4} - \frac{1}{16} 
= \frac{3}{16},
$$
as claimed.

Repeating this argument for general $k$, we can obtain the same bounds for 
$\abs{\ip{w}{v_k}}$. 
Finally, take the union bound over the $K$ choices of $k$. The random vector
satisfies the conclusion with probability at least $1-2K^{-7}>0$. 
The proof is complete.
\end{proof}

Lemma~\ref{lem: perception 1} essentially says that one neuron can memorize  
the labels of $O(d)$ data points in $\R^d$. Thus, $r$ neurons should be able to 
memorize the labels of $O(dr)$ data points in $\R^d$. To make this happen, 
we can partition the data into $r$ batches of size $O(d)$ each, and
train each neuron on a different batch. The following lemma makes this formal; 
to see the connection, apply it for $\eta \asymp 1/\sqrt{d}$.

\begin{theorem}[One layer]				\label{thm: one layer}
  Consider a number $\eta \in (0,1)$, 
  vectors $v_1,\ldots,v_K \in \R^d$ 
  and labels $y_1,\ldots,y_K \in \{0,1\}$
  as in Lemma~\ref{lem: perception 1}. 
  Assume that
  $(2K_1+r) \log K \le c_4 r \eta^{-2}$
  where $r$ is a positive integer.
  Then there exists a pseudolinear map $P: \R^d \to \R^r$
  such that for all $k=1,\ldots,K$ we have:
  $$
  P(v_k)=0 \text{ iff } y_k=0.
  $$
\end{theorem}

\begin{proof}
Without loss of generality, assume that the first $K_1$ of the labels $y_k$ equal $1$ 
and the rest equal zero, i.e. $y_k = \one_{\{k \le K_1\}}$.
Partition the indices of the nonzero labels $\{1,\ldots,K_1\}$ into $r/2$ subsets $I_i$ (``batches''),
each of cardinality at most $2K_1/r+1$. For each batch $i$, define a new set of labels 
$$
y_{ki} = \one_{\{k \in I_i\}}, \quad k=1,\ldots,K.
$$
In other words, the labels $y_{ki}$ are obtained from the original labels $y_i$ 
by zeroing out the labels outside batch $i$.

For each $i$, apply Lemma~\ref{lem: perception 1} for the labels $y_{ki}$.
The number of nonzero labels is $\abs{I_i} \le 2K_1/r+1$, so we can use this number instead of 
$K_1$, noting that the condition $(2K_1/r+1) \log K \le c_4 \eta^{-2}$ required in Lemma~\ref{lem: perception 1} does hold by our assumption. 
We obtain a vector $w_i \in \R^d$ that satisfies 
the following holds for every $k=1,\ldots,K$:
\begin{equation}	\label{eq: batch perception}
\abs{\ip{w_i}{v_k}} \le \frac{1}{16} \text{ if $k \not\in I_i$}; \qquad
\abs{\ip{w_i}{v_k}} \ge \frac{3}{16} \text{ if $k \in I_i$}.
\end{equation}

Define the pseudolinear map $\Phi(v) = \big( \Phi(v)_1,\ldots, \Phi(v)_r \big)$ as follows:
$$
P(v)_i := \phi \Big( \ip{w_i}{v} - \frac{1}{8} \Big), \quad
P(v)_{r/2+i} := \phi \Big( -\ip{w_i}{v} - \frac{1}{8} \Big), \quad
i=1,\ldots,r/2.
$$

If $y_k=0$, then $k>K_1$. Thus $k$ does not belong to any batch $I_i$, 
and \eqref{eq: batch perception} implies that $\abs{\ip{w_i}{v_k}} \le 1/16$ for all $i$.
Then both $\ip{w_i}{v_k} - 1/8$ and $-\ip{w_i}{v_k} - 1/8$ are negative, and 
since $\phi(t)=0$ for negative $t$, all coordinates of $P(v_k)$ are zero, 
i.e. $P(v_k)=0$.

Conversely, if $P(v_k)=0$ then, by construction, for each $i$ both $\ip{w_i}{v} - 1/8$ 
and $\ip{w_i}{v} - 1/8$ must be nonpositive, which yields
$\abs{\ip{w_i}{v_k}} \le 1/8 < 3/16$. Thus, by \eqref{eq: batch perception}, $k$ may 
not belong to any batch $I_i$, which means that $k>K_1$, and this implies $y_k=0$.
\end{proof}

\section{Assembly}			\label{s: assembly}

In this section we prove a general version of our main result. 
Let us first show how to train a network with four layers.
To this end, choose an enrichment map from layer $1$ to layer $2$ to transform 
the data from merely separated to $\e$-orthogonal, 
choose a map from layer $2$ to layer $3$ 
to further enrich the data by making it $O(1/\sqrt{d})$-orthogonal, 
and finally make a map from layer $3$ to layer $4$ memorize the labels.
This yields the following result:

\begin{theorem}[Shallow networks]				\label{thm: shallow}
  Consider unit vectors $x_1,\ldots,x_K \in \R^n$ that satisfy
  $$
  \|x_i-x_j\|_2 \ge C_2 \sqrt{\frac{\log\log d}{\log m}}.
  $$
  Consider any labels $y_1,\ldots,y_K \in \{0,1\}$, at most $K_1$ of which equal $1$. 
  Assume that
  $$
  K_1 \log^5(dK) \le c_5 dr
  $$
  as well as $K \le \exp(c_5 m^{1/5})$, $K \le \exp(c_5 d^{1/5})$, and $d \le \exp(c_5 m^{1/5})$.  
  Then there exists a map $F \in \FF(n,m,d,r)$     
  such that for all $k=1,\ldots,K$ we have:
  $$
  F(x_k)=0 \text{ iff } y_k=0.
  $$
\end{theorem}

\begin{proof}
{\em Step 1. From separated to $\e$-orthogonal.}
Apply Theorem~\ref{thm: separated to eps-orthogonal} with $\e=c_5/\log d$.
(Note that the required constraints in that result hold by our assumptions with small $c_5$.)
We obtain an almost pseudolinear map $E: \R^n \to \R^m$ such that 
the vectors $u_k := E(x_k)$ satisfy
$$
\abs[1]{\|u_i\|_2^2-1} \le \e, \quad   
\abs[1]{\ip{u_i}{u_j}} \le \e
$$
for all distinct $i,j$.

{\em Step 2. From $\e$-orthogonal to $\frac{1}{\sqrt{d}}$-orthogonal.}
Apply Theorem~\ref{thm: eps-orthogonal to max orthogonal}. 
We obtain an almost pseudolinear map $R: \R^m \to \R^d$ 
such that the vectors $v_k := R(u_k)$ satisfy
$$
\|v_i\|_2 \ge \frac{1}{2}, \quad   
\abs[1]{\ip{v_i}{v_j}} 
\le \frac{C_4\log^2 (dK)}{\sqrt{d}} 
=: \eta
$$
for all distinct indices $i,j$.

{\em Step 3. Perception.}
Apply Theorem~\ref{thm: one layer}. 
(Note that our assumptions with small enough $c_5$ ensure that the required constraint 
$(2K_1+r) \log K \le c_4r\eta^{-2}$ does hold.)
We obtain a pseudolinear map $P: \R^d \to \R^r$
such that the vectors $z_k := P(v_k)$ satisfy: 
$$
z_k=0 \text{ iff } y_k=0.
$$

{\em Step 4. Assembly.}
Define
$$
F := P \circ R \circ E.
$$
Since $E$ and $R$ are almost pseudolinear and $P$ is pseudolinear, 
$F$ can be represented as a composition of three pseudolinear maps
(by absorbing the linear factors), i.e. $F \in \FF(n,m,d,r)$.
Also, $F(x_k) = z_k$ by construction, so the proof is complete.
\end{proof}

Finally, we can extend Theorem~\ref{thm: shallow} for arbitrarily deep networks 
by distributing the memorization tasks among all layers evenly. 
Indeed, consider a network with $L$ layers and with $n_i$ nodes in layer $i$.
As in Theorem~\ref{thm: shallow}, first we enrich the data, thereby 
making the input to layer $3$ almost orthogonal. 
Train the map from layer $3$ to layer $4$ to memorize the labels of the first 
$O(n_3 n_4)$ data points 
using Theorem~\ref{thm: one layer} (for $d=n_3$, $r=n_4$, and $\eta \asymp 1/\sqrt{d}$).
Similarly, train the map from layer $4$ to layer $5$ to memorize the labels of the next
$O(n_4 n_5)$ data points, and so on. This allows us to train the network on the total of  
$O(n_3 n_4+n_4n_5+\cdots+n_{L-1}n_L) = O(W)$ data points, where $W$ is the 
number of ``deep connections'' in the network, i.e. connections that occur from layer $3$ and up. 
This leads us to the main result of this paper.

\begin{theorem}[Deep networks]			\label{thm: deep}
  Let $n_1,\ldots,n_L$ be positive integers, and set 
  $n_0 := \min(n_2,\ldots,n_L)$ and $n_\infty := \max(n_2,\ldots,n_L)$.
  Consider unit vectors $x_1,\ldots,x_K \in \R^n$ that satisfy
  $$
  \|x_i-x_j\|_2 \ge C \sqrt{\frac{\log\log n_\infty}{\log n_0}}.
  $$
  Consider any labels $y_1,\ldots,y_K \in \{0,1\}$, at most $K_1$ of which equal $1$. 
  Assume that the number of deep connections $W := n_3 n_4 + \cdots n_{L-1} n_L$
  satisfies
  \begin{equation}	\label{eq: W deep}
  W \ge C K_1 \log^5 K,
  \end{equation}
  as well as $K \le \exp(cn_0^{1/5})$ and $n_\infty \le \exp(cn_0^{1/5})$.
  Then there exists a map $F \in \FF(n_1,\ldots,n_L)$ 
  such that for all $k=1,\ldots,K$ we have:
  $$
  F(x_k)=0 \text{ iff } y_k=0.
  $$
\end{theorem}

We stated a simplified version of this result in Theorem~\ref{thm: main}. 
To see the connection, just take the `OR' of the outputs of all $n_L$ nodes 
of the last layer.

\begin{proof}
{\em Step 1. Initial reductions.}
For networks with $L=4$ layers, we already proved the result in Theorem~\ref{thm: shallow},
so we an assume that $L \ge 5$. 
Moreover, for $K=1$ the result is trivial, so we can assume that $K \ge 2$. 
In this case, if we make the constant $c$ in our assumption $2 \le \exp(cn_0^{1/5})$ 
sufficiently small, we can assume that $n_0$ (and thus also all $n_i$ and $W$)
are arbitrarily large, i.e. larger than any given absolute constant.

\medskip
{\em Step 2. Distributing data to layers.}
Without loss of generality, assume that the first $K_1$ of the labels $y_k$ equal $1$ 
and the rest equal zero, i.e. $y_k = \one_{\{k \le K_1\}}$.
Partition the indices of the nonzero labels $\{1,\ldots,K_1\}$ into subsets $I_3, \ldots, I_{L-1}$
(``batches'') so that
$$
\abs{I_i} \le \frac{n_i n_{i+1}}{W} K_1 + 1.
$$
(This is possible since the numbers $n_i n_{i+1}/W$ sum to one.)
For each batch $i$, define a new set of labels 
$$
y_{ki} = \one_{\{k \in I_i\}}, \quad k=1,\ldots,K.
$$
In other words, $y_{ki}$ are obtained from the original labels $y_i$ 
by zeroing out the labels outside batch $i$.

\medskip
{\em Step 3. Memorization at each layer.}
For each $i$, apply Theorem~\ref{thm: shallow} for the labels $y_{ki}$, 
for the number of nonzero labels $\abs{I_i}$ instead of $K_1$, 
and for $n=n_1$, $m=n_0/3$, $d=n_i/3$, $r=n_{i+1}/3$.
Thus, if 
\begin{equation}	\label{eq: niK1}
\bigg( \frac{n_i n_{i+1}}{W} K_1 + 1 \bigg) \log^5 \bigg( \frac{n_i K}{3} \bigg) 
\le \frac{c_5 n_i n_{i+1}}{9}
\end{equation}
as well as 
\begin{equation}	\label{eq: KKni}
K \le \exp(c_5 n_0^{1/5}/3), \qquad
K \le \exp(c_5 n_i^{1/5}/3), \qquad
n_i \le \exp(c_5 n_0^{1/5}/3),
\end{equation}
then there exist a map 
$$
F_i \in \FF(n_1,n_0/3,n_i/3,n_{i+1}/3)
$$ 
satisfying the following for all $i$ and $k$:
\begin{equation}	\label{eq: Fixk}
F_i(x_k)=0 \text{ iff } y_{ki}=0.
\end{equation}
Moreover, when we factorize $F_i = P_i \circ R_i \circ E_i$ into three almost pseudolinear maps, 
then $E_i=E$, the enrichment map from $\R^{n_1}$ into $\R^{n_0/3}$,
is trivially independent of $i$, so 
\begin{equation}	\label{eq: Fi factorized}
F_i = P_i \circ R_i \circ E.
\end{equation}

Our assumptions with sufficiently small $c$ guarantee 
that the required conditions \eqref{eq: KKni} do hold. 
In order to check \eqref{eq: niK1}, we will first note a somewhat stronger bound  
than \eqref{eq: W deep} holds, namely we have
\begin{equation}	\label{eq: W deep stronger}
3^5 W \ge CK_1 \log^5 (n_i K), \quad i=3,\ldots,L-1.
\end{equation}
Indeed, if $W \ge K^2$ then using that $K_1 \le K \le W^{1/2}$ and $n_i \le W$ we get
$$
K_1 \log^5 (n_i K)
\le W^{1/2} \log^5(W^{3/2})
= \frac{3^5}{2^5} W^{1/2} \log^5 W
\le \frac{3^5}{C} W
$$
when $W$ is sufficiently large. 
If $W \le K^2$ then using that $n_i \le W \le K^2$ we get 
$$
K_1 \log^5 (n_i K)
\le K_1 \log^5(K^3)
= 3^5 K_1 \log^5 K
\le \frac{3^5}{C} W
$$
where the last step follows from \eqref{eq: W deep}.
Hence, we verified \eqref{eq: W deep stronger} for the entire range of $W$.

Now, to check \eqref{eq: niK1}, note that
$$
\log^5 (n_i K)
\le 2^5 \big( \log^5 n_i + \log^5 K \big)
\le \frac{c_5 n_i}{20}
\le  \frac{c_5 n_i n_{i+1}}{20}
$$
where we used that $n_i$ is arbitrarily large and the assumption on $K$
with a sufficiently small constant $c$. 
Combining this bound with \eqref{eq: W deep stronger}, we obtain 
$$
\bigg( \frac{n_i n_{i+1}}{W} K_1 + 1 \bigg) \log^5 \bigg( \frac{n_i K}{3} \bigg) 
\le \bigg( \frac{3^5}{C} + \frac{c_5}{20} \bigg) n_i n_{i+1}
\le \frac{c_5 n_i n_{i+1}}{9}
$$
if $C$ is sufficiently large. We have checked\eqref{eq: niK1}.

\medskip
{\em Step 4. Stacking.}
To complete the proof, it suffices to construct map 
$F  \in \FF(n_1,\ldots,n_L)$ with the following property:
\begin{equation}	\label{eq: FxFix}
F(x) = 0 \quad \text{iff} \quad F_i(x)=0 \; \forall i=3,\ldots,L-1.
\end{equation}
Indeed, this would imply that $F(x_k)=0$ happens iff
$F_i(x_k)=0$ for all $i$, which, according to \eqref{eq: Fixk} is equivalent
to $y_{ki}=0$ for all $i$. By definition of $y_{ki}$, this is further equivalent
to $k \not\in I_i$ for any $i$, which by construction of $I_i$ is equivalent
to $k>K_1$, which is finally equivalent to $y_k=0$, as claimed. 

\begin{figure}[htp]	
  \centering 
    \includegraphics[width=0.7\textwidth]{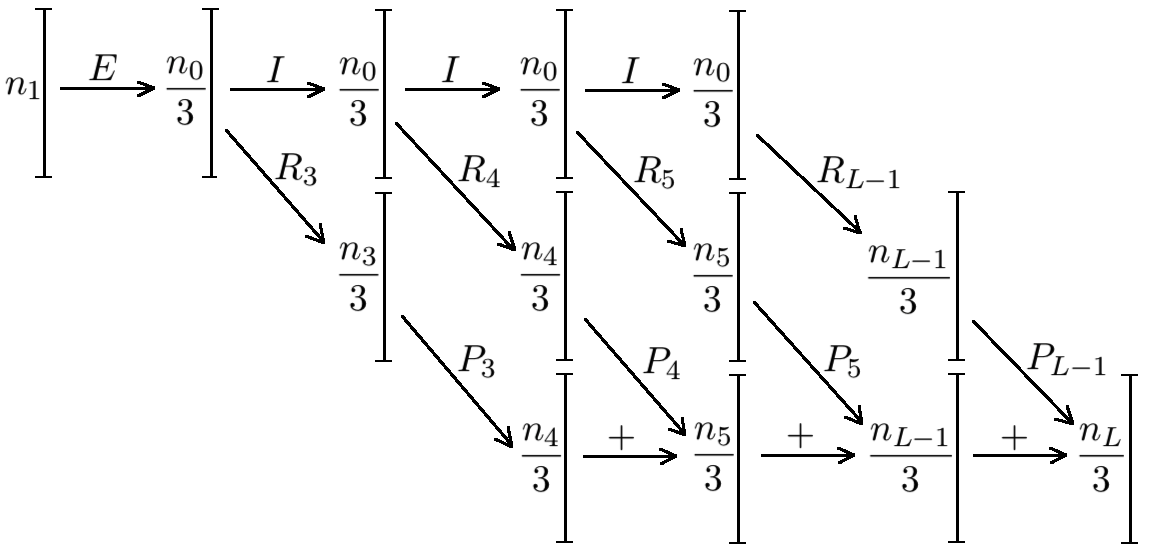} 
    \caption{Trading width for depth: stacking shallow networks into a deep network.}
  \label{fig: assembly}
\end{figure}

We construct $F$ by ``stacking'' the maps $F_i = P_i \circ R_i \circ E$ for $i=3,\ldots,L-1$ 
as illustrated in Figure~\ref{fig: assembly}. 
To help us stack these maps, we drop some nodes from the original network and 
first construct 
$$
F \in \FF(n'_1, \ldots, n'_L)
$$
with some $n'_i \le n_i$; we can then extend $F$ trivially to $\FF(n_1, \ldots, n_L)$.
As Figure~\ref{fig: assembly} suggests, we choose
$n'_1=n_1$, 
$n'_2=n_0/3$, 
$n'_3=n_0/3+n_3/3$, 
$n'_i=n_0/3+2n_i/3$ for $i=4,\ldots,L-2$ (skip these layers if $L=5$),
$n'_{L-1} = 2n_{L-1}/3$, 
and $n'_L = n_L/3$.
Note that by definition of $n_0$, we indeed have $n'_i \le n_i$ for all $i$.

We made this choice so that the network can realize the maps $F_i$.
As Figure~\ref{fig: assembly} illustrates, 
the map $F_3 = P_3 \circ R_3 \circ E \in \FF(n_1,n_0/3,n_3/3,n_4/3)$ is realized 
by setting the factor $E : \R^{n_1} \to \R^{n_0/3}$
to map the first layer to the second, the factor $R_3: \R^{n_0/3} \to \R^{n_3/3}$ to map the second 
layer to the last $n_3/3$ nodes of the third layer, and the factor $P_3: \R^{n_3/3} \to \R^{n_4/3}$ 
to map further to the last $n_3/3$ nodes of the fourth layer.
Moreover, the output of the second layer is transferred to the first $n_0/3$ nodes of the third layer by the identity map\footnote{Note that the identity map restricted to the image of $E$ can be realized
as an almost pseudolinear map for both ReLU and threshold nonlinearities. 
For ReLU this is obvious by setting the bias large enough; 
for threshold nonlinearity note that the image of the almost pseudolinear map $E$ consists of vectors whose coordinates are either zero or take the same value $\l$. Thus, the Heaviside function multiplied by $\l$ is the identity on the image of $E$.} $I$, so we can realize the next map $F_4$, and so on.

The outputs of all maps $F_i$ are {\em added together} as the signs ``+'' in Figure~\ref{fig: assembly} indicate. Namely, the components of the output of $F_1$, i.e. the last $n_4/3$ nodes of the fourth layer, are summed together and added to any node (say, the last node) 
of the fifth layer; the components of the output of $F_2$, i.e. the last $n_5/3$ nodes of the fourth layer, are summed together and added to the last node of the sixth layer, and so on. 
For ReLU nonlinearity, the $+$ refers to addition of real numbers; 
for threshold nonlinearity, we replace adding by taking the maximum (i.e. the `OR' operation), 
which is clearly realizable. 

\medskip
{\em Step 5. Conclusion.}
Due to our construction, the sum of all $n'_L$ components of the function $F(x)$
computed by the network equals the sum (or maximum, for threshold nonlinearity) 
of all components of all functions 
$F_i(x)$. Since the components are always nonnegative, 
$F(x)$ is zero iff all components of all functions $F_i(x)$ are zero. In other words, 
our claim \eqref{eq: FxFix} holds. 
\end{proof}

\appendix
\section{Asymptotical expressions for Gaussian integrals}

The asymptotical expansion of Mills ratio for the normal distribution is 
well known, see \cite{pinelis2002monotonicity}. For our purposes, 
the first three terms of the expansion will be sufficient:
\begin{equation}	\label{eq: Mills}
\Psi(a) 
= \frac{\int_a^\infty e^{-x^2/2} \; dx}{e^{-a^2/2}}
= a^{-1} - a^{-3} + 3a^{-5} + O(a^{-7}).
\end{equation}
In particular, the tail probability of the standard normal random variable 
$\gamma \sim N(0,1)$ satisfies
\begin{equation}	\label{eq: Mills ratio tail}
\Pr{\gamma>a} = \frac{1}{\sqrt{2\pi}} e^{-a^2/2} \Big( a^{-1} + O(a^{-3}) \Big).
\end{equation}

The following two lemmas give asymptotical expressions
for the expected value of the first two moments of 
the random variable $(\gamma-a)_+ = \max(\gamma-a,0)$ where, as before, 
$\gamma \sim N(0,1)$ is standard normal. 

\begin{lemma}[ReLU of the normal distribution]		\label{lem: ReLU normal}
  Let $\gamma \sim N(0,1)$. Then, as $a \to \infty$, we have
  \begin{align*}
  \E (\gamma-a)_+ &= \frac{1}{\sqrt{2\pi}} e^{-a^2/2} \Big( a^{-2} + O(a^{-4}) \Big), \\
  \E ((\gamma-a)_+)^2 &= \frac{1}{\sqrt{2\pi}} e^{-a^2/2} \Big( 2a^{-3} + O(a^{-5}) \Big). 
  \end{align*}
\end{lemma}

\begin{proof}
Expressing expectation as the integral of the tail (see e.g. \cite[Lemma~1.2.1]{Vbook}), we have
\begin{align*} 
\sqrt{2\pi} \, \E (\gamma-a)_+
&= \int_0^\infty (x-a)_+ \, e^{-x^2/2} \; dx
=  \int_a^\infty (x-a) e^{-x^2/2} \; dx \\
&= \int_a^\infty x e^{-x^2/2} \; dx - a \int_a^\infty e^{-x^2/2} \; dx.
\end{align*}
Using substitution $y=x^2/2$, we see that the value of the 
first integral on the right hand side is $e^{-a^2/2}$. 
Using the Mills ratio asymptotics \eqref{eq: Mills} for the second integral, we get
$$
\sqrt{2\pi} \, \E (\gamma-a)_+
= e^{-a^2/2} - a \cdot e^{-a^2/2} \Big( a^{-1} - a^{-3} + O(a^{-5}) \Big)
= e^{-a^2/2} \Big( a^{-2} + O(a^{-4}) \Big).
$$
This finishes the first part of the lemma.

To prove the second part, we start similarly: 
\begin{align*} 
\sqrt{2\pi} \, \E ((\gamma-a)_+)^2
&=  \int_a^\infty (x-a)^2 e^{-x^2/2} \; dx  \\
&= \int_a^\infty x^2 e^{-x^2/2} \; dx 
	- 2a \int_a^\infty x e^{-x^2/2} \; dx
	+ a^2 \int_a^\infty e^{-x^2/2} \; dx.
\end{align*}
Integrating by parts, we find that the first integral on the right side
equals 
$$
a e^{-a^2/2} + \int_a^\infty e^{-x^2/2} \; dx
= a e^{-a^2/2} + \Psi(a) e^{-a^2/2};
$$
the second integral equals $e^{-a^2/2}$ as before, 
and the third integral equals $\Psi(a) e^{-a^2/2}$.
Combining these and using the asymptotical expansion \eqref{eq: Mills} 
for $\Psi(a)$, we conclude that
\begin{align*} 
\sqrt{2\pi} \, \E ((\gamma-a)_+)^2
&= a e^{-a^2/2} + \Psi(a) e^{-a^2/2} 
	-2a e^{-a^2/2}
	+ a^2 \Psi(a) e^{-a^2/2} \\
&= e^{-a^2/2} \Big( (a^2+1) \Psi(a) - a \Big)
= e^{-a^2/2} \Big( 2a^{-3} + O(a^{-5}) \Big).
\end{align*}
This completes the proof of the second part of the lemma.
\end{proof}

\begin{lemma}[Stability]			\label{lem: stability}
  Fix any $z > -1$. Let $\gamma \sim N(0,1)$. Then, as $a \to \infty$, we have 
  \begin{align}
  \frac{\P\{\gamma \sqrt{1+z} > a\}}{\Pr{\gamma>a}}
  	&= \exp \bigg( \frac{a^2 z}{2(1+z)} \bigg) \, (1+z)^{1/2} \, \Big( 1 + O(a^{-2}) \Big); \\
  \frac{\E (\gamma \sqrt{1+z} - a)_+}{\E(\gamma - a)_+}
  	&= \exp \bigg( \frac{a^2 z}{2(1+z)} \bigg) \, (1+z) \, \Big( 1 + O(a^{-2}) \Big); \\
  \frac{\E \big( (\gamma \sqrt{1+z} - a)_+ \big)^2}{\E((\gamma - a)_+)^2}
  	&= \exp \bigg( \frac{a^2 z}{2(1+z)} \bigg) \, (1+z)^{3/2} \, \Big( 1 + O(a^{-2}) \Big).
  \end{align} 
\end{lemma}

\begin{proof}
Use the asymptotics in \eqref{eq: Mills ratio tail} and Lemma~\ref{lem: ReLU normal} 
for $a$ and $a/\sqrt{1+z}$ and simplify.
\end{proof}

We complete this paper by proving an elementary monotonicity property for Gaussian integrals, 
which we used in the proof of of Proposition~\ref{prop: correlation epsilon-orthogonal}.

\begin{lemma}[Monotonicity]				\label{lem: monotonicity}
  Let $\psi: \R \to [0,\infty)$ be a nondecreasing function satisfying $\psi(t)=0$ for all $t<0$,
  and let $\gamma \sim N(0,1)$. Then $\sigma \mapsto \E \psi(\s\gamma)$ is a nondecreasing 
  function on $[0,\infty)$.
\end{lemma}

\begin{proof}
Denoting by $f(x)$ the probability density function of $N(0,1)$, we have
$$
\E \psi(\s\gamma)
= \int_\infty^\infty \psi(\s x) f(x) \, dx
= \int_0^\infty \psi(\s x) f(x) \, dx.
$$
The last step follows since, by assumption, $\psi(\s x)=0$ for all $x<0$.
To complete the proof, it remains to note that for every fixed $x \ge 0$, 
the function $\sigma \mapsto \psi(\s x)$ is nondecreasing. 
\end{proof}

\bibliographystyle{plain}
\bibliography{neural-networks}

\begin{thebibliography}{10}

\bibitem{allen2018convergence}
Zeyuan Allen-Zhu, Yuanzhi Li, and Zhao Song.
\newblock A convergence theory for deep learning via over-parameterization.
\newblock {\em arXiv preprint arXiv:1811.03962}, 2018.

\bibitem{arora2019fine}
Sanjeev Arora, Simon~S Du, Wei Hu, Zhiyuan Li, and Ruosong Wang.
\newblock Fine-grained analysis of optimization and generalization for
  overparameterized two-layer neural networks.
\newblock {\em arXiv preprint arXiv:1901.08584}, 2019.

\bibitem{baldi2018neuronal}
Pierre Baldi and Roman Vershynin.
\newblock On neuronal capacity.
\newblock In {\em Advances in Neural Information Processing Systems}, pages
  7729--7738, 2018.

\bibitem{baldi2019capacity}
Pierre Baldi and Roman Vershynin.
\newblock The capacity of feedforward neural networks.
\newblock {\em Neural Networks}, 116:288--311, 2019.

\bibitem{bartlett2019nearly}
Peter~L Bartlett, Nick Harvey, Christopher Liaw, and Abbas Mehrabian.
\newblock Nearly-tight vc-dimension and pseudodimension bounds for piecewise
  linear neural networks.
\newblock {\em Journal of Machine Learning Research}, 20(63):1--17, 2019.

\bibitem{baum1988capabilities}
Eric~B Baum.
\newblock On the capabilities of multilayer perceptrons.
\newblock {\em Journal of complexity}, 4(3):193--215, 1988.

\bibitem{baum1989size}
Eric~B Baum and David Haussler.
\newblock What size net gives valid generalization?
\newblock In {\em Advances in neural information processing systems}, pages
  81--90, 1989.

\bibitem{cover1965geometrical}
Thomas~M Cover.
\newblock Geometrical and statistical properties of systems of linear
  inequalities with applications in pattern recognition.
\newblock {\em IEEE transactions on electronic computers}, (3):326--334, 1965.

\bibitem{du2018gradientb}
Simon~S Du, Jason~D Lee, Haochuan Li, Liwei Wang, and Xiyu Zhai.
\newblock Gradient descent finds global minima of deep neural networks.
\newblock {\em arXiv preprint arXiv:1811.03804}, 2018.

\bibitem{du2018gradienta}
Simon~S Du, Xiyu Zhai, Barnabas Poczos, and Aarti Singh.
\newblock Gradient descent provably optimizes over-parameterized neural
  networks.
\newblock {\em arXiv preprint arXiv:1810.02054}, 2018.

\bibitem{ge2019mildly}
Rong Ge, Runzhe Wang, and Haoyu Zhao.
\newblock Mildly overparametrized neural nets can memorize training data
  efficiently.
\newblock {\em arXiv preprint arXiv:1909.11837}, 2019.

\bibitem{hardt2016identity}
Moritz Hardt and Tengyu Ma.
\newblock Identity matters in deep learning.
\newblock {\em arXiv preprint arXiv:1611.04231}, 2016.

\bibitem{huang2003learning}
Guang-Bin Huang.
\newblock Learning capability and storage capacity of two-hidden-layer
  feedforward networks.
\newblock {\em IEEE Transactions on Neural Networks}, 14(2):274--281, 2003.

\bibitem{ji2019polylogarithmic}
Ziwei Ji and Matus Telgarsky.
\newblock Polylogarithmic width suffices for gradient descent to achieve
  arbitrarily small test error with shallow relu networks.
\newblock {\em arXiv preprint arXiv:1909.12292}, 2019.

\bibitem{kowalczyk1997estimates}
Adam Kowalczyk.
\newblock Estimates of storage capacity of multilayer perceptron with threshold
  logic hidden units.
\newblock {\em Neural networks}, 10(8):1417--1433, 1997.

\bibitem{li2018learning}
Yuanzhi Li and Yingyu Liang.
\newblock Learning overparameterized neural networks via stochastic gradient
  descent on structured data.
\newblock In {\em Advances in Neural Information Processing Systems}, pages
  8157--8166, 2018.

\bibitem{mitchison1989bounds}
GJ~Mitchison and RM~Durbin.
\newblock Bounds on the learning capacity of some multi-layer networks.
\newblock {\em Biological Cybernetics}, 60(5):345--365, 1989.

\bibitem{oymak2019towards}
Samet Oymak and Mahdi Soltanolkotabi.
\newblock Towards moderate overparameterization: global convergence guarantees
  for training shallow neural networks.
\newblock {\em arXiv preprint arXiv:1902.04674}, 2019.

\bibitem{pinelis2002monotonicity}
Iosif Pinelis.
\newblock Monotonicity properties of the relative error of a pad{\'e}
  approximation for mills’ ratio.
\newblock {\em J. Inequal. Pure Appl. Math}, 3(2):1--8, 2002.

\bibitem{song2019quadratic}
Zhao Song and Xin Yang.
\newblock Quadratic suffices for over-parametrization via matrix chernoff
  bound.
\newblock {\em arXiv preprint arXiv:1906.03593}, 2019.

\bibitem{sun2019optimization}
Ruoyu Sun.
\newblock Optimization for deep learning: theory and algorithms.
\newblock {\em arXiv preprint arXiv:1912.08957}, 2019.

\bibitem{Vbook}
Roman Vershynin.
\newblock High-dimensional probability: An introduction with applications in
  data science. cambridge series in statistical and probabilistic mathematics,
  2018.

\bibitem{yamasaki1993lower}
Masami Yamasaki.
\newblock The lower bound of the capacity for a neural network with multiple
  hidden layers.
\newblock In {\em International Conference on Artificial Neural Networks},
  pages 546--549. Springer, 1993.

\bibitem{yun2019small}
Chulhee Yun, Suvrit Sra, and Ali Jadbabaie.
\newblock Small relu networks are powerful memorizers: a tight analysis of
  memorization capacity.
\newblock In {\em Advances in Neural Information Processing Systems}, pages
  15532--15543, 2019.

\bibitem{zhang2016understanding}
Chiyuan Zhang, Samy Bengio, Moritz Hardt, Benjamin Recht, and Oriol Vinyals.
\newblock Understanding deep learning requires rethinking generalization.
\newblock {\em ICLR}, 2017.

\bibitem{zou2018stochastic}
Difan Zou, Yuan Cao, Dongruo Zhou, and Quanquan Gu.
\newblock Stochastic gradient descent optimizes over-parameterized deep relu
  networks.
\newblock {\em arXiv preprint arXiv:1811.08888}, 2018.

\bibitem{zou2019improved}
Difan Zou and Quanquan Gu.
\newblock An improved analysis of training over-parameterized deep neural
  networks.
\newblock {\em arXiv preprint arXiv:1906.04688}, 2019.

\end{thebibliography}

\end{document}